\documentclass{article}
\PassOptionsToPackage{numbers}{natbib}

\usepackage[final]{neurips_2021}

\usepackage[utf8]{inputenc} 
\usepackage[T1]{fontenc}    
\usepackage[pdfencoding=auto, psdextra]{hyperref}
\usepackage{url}            
\usepackage{booktabs}       
\usepackage{amsfonts}       
\usepackage{nicefrac}       
\usepackage{microtype}      
\usepackage{xcolor}         
\usepackage[T1]{fontenc}

\usepackage[utf8]{inputenc} 
\usepackage[T1]{fontenc} 
\usepackage{mathtools}
\usepackage{commands}
\usepackage{multirow}
\usepackage{tablefootnote}
\usepackage{xcolor} 
\usepackage{nicefrac}
\usepackage{amsfonts}
\usepackage{url}
\usepackage{caption}
\usepackage{wrapfig}
\usepackage[outline]{contour}
\usepackage{makecell}
\usepackage[symbol]{footmisc} 

\DeclareRobustCommand{\sbseries}{\fontseries{sb}\selectfont}
\DeclareTextFontCommand{\textsb}{\sbseries}

\title{Two Sides of Meta-Learning Evaluation: \\In vs. Out of Distribution}

\author{%
Amrith Setlur$^{1}$\thanks{Authors contributed equally to this paper.} \qquad Oscar Li$^{2*}$ \qquad Virginia Smith$^{2}$ \\
  \footnotesize \texttt{asetlur@cs.cmu.edu} \quad \texttt{oscarli@cmu.edu}     \quad \texttt{smithv@cmu.edu}\quad \quad\\
    \small ${^1}$Language Technologies Institute \qquad $^2$Machine Learning Department\\ 
    \small School of Computer Science, Carnegie Mellon University \\
}

\begin{document}

\maketitle
\begin{abstract}
We categorize meta-learning evaluation into two settings: \textit{in-distribution} [ID], in which the train and test tasks are sampled \iid from the same underlying task distribution, and \textit{out-of-distribution} [OOD], in which they are not. While most meta-learning theory and some FSL applications follow the ID setting, we identify that most existing few-shot classification benchmarks instead  reflect OOD evaluation, as they use disjoint sets of train (base) and test (novel) classes for task generation. This discrepancy is problematic because---as we show on numerous benchmarks---meta-learning methods that perform better on existing OOD datasets may perform significantly worse in the ID setting. In addition, in the OOD setting, even though current FSL benchmarks seem befitting, our study highlights concerns in 1) reliably performing model selection for a given meta-learning method, and 2) consistently comparing the performance of different methods. To address these concerns, we provide suggestions on how to construct FSL benchmarks to allow for ID evaluation as well as more reliable OOD evaluation. Our work\footnote{Code available at \href{https://github.com/ars22/meta-learning-eval-id-vs-ood}{https://github.com/ars22/meta-learning-eval-id-vs-ood}.} aims to inform the meta-learning community about the importance and distinction of ID vs. OOD evaluation, as well as the subtleties of OOD evaluation with current  benchmarks.
\end{abstract}

\vspace{-6pt}
\vspace{-4pt}
\section{Introduction}
\label{sec:introduction}
\vspace{-6pt}

Meta-learning considers learning algorithms that can perform well over a distribution of tasks~\cite{hospedales2020meta,schmidhuber1987evolutionary}. To do so, a meta-learning method first learns from a set of tasks sampled from a training task distribution (\textit{meta-training}), and then evaluates the quality of the learned algorithm using tasks from a test task distribution (\textit{meta-testing}). The test task distribution can be the same as the training task distribution (a scenario we term \textit{in-distribution} generalization evaluation or ID evaluation) or a different task distribution (\textit{out-of-distribution} generalization evaluation or OOD evaluation).

In this work, we argue that there is a need to carefully consider current meta-learning practices in light of this ID vs. OOD categorization. In particular, meta-learning is commonly evaluated on few-shot learning (FSL) benchmarks, which aim to evaluate meta-learning methods’ ability to learn sample-efficient algorithms. Current benchmarks primarily focus on image classification and provide training tasks constructed from a set of train (base) classes that are completely disjoint and sometimes extremely different from the test (novel) classes used for test tasks. As we discuss in Section~\ref{sec:background}, this design choice imposes a natural shift in the train and test task distribution that makes current benchmarks reflective of OOD generalization. 
However, there are a number of reasons to also consider the distinct setting of ID evaluation. First, whether in terms of methodology or theory, many works motivate and analyze meta-learning under the assumption that train and test tasks are sampled \textit{iid} from the same distribution (see Section~\ref{sec:relwork}). Second, we identify a growing number of applications, such as federated learning, where there is in fact a need for sample-efficient algorithms that can perform ID generalization. Crucially, we show across numerous benchmarks that methods that perform well OOD may perform significantly worse in ID settings. Our results highlight that it is critical to clearly define which setting a researcher is targeting when developing new meta-learning methods, and we provide tools for modifying existing benchmarks to reflect both scenarios.

Beyond this, we also re-examine current OOD FSL benchmarks and analyze how the shift in the train and test task distributions may impact the reliability of OOD evaluations. We point out two  concerns which we believe are not widely considered in the meta-learning community. 
First, unlike areas such as domain generalization where model selection challenges are more widely discussed~\cite{gulrajani2020search,koh2020wilds}, we conduct to the best of our knowledge the first rigorous study demonstrating the difficulty of model selection due to the shift in the validation and test task distributions in FSL benchmarks.
Second, because the OOD scenario in meta-learning does not assume a specific test task distribution, there is room for different test distributions to be used for evaluation. We show that comparing which meta-learning method performs better can be unreliable not only over different OOD FSL benchmarks, but also within a single benchmark depending on the number of novel classes.

Our main contributions are:
\textbf{i)} We clearly outline both ID and OOD FSL evaluation scenarios and explain why most popular FSL benchmarks target OOD evaluation (Section~\ref{sec:background}).
\textbf{ii)} We provide realistic examples of the ID scenario and show that the performance of popular meta-learning methods can drastically differ in ID vs. OOD scenarios (Section~\ref{sec:ID}).
\textbf{iii)} For existing OOD FSL benchmarks, we highlight concerns with a) current model selection strategies for meta-learning methods, and b) the reliability of meta-learning method comparisons (Section~\ref{sec:OOD}).
\textbf{iv)} To remedy these concerns, we suggest suitable modifications to the current FSL benchmarks to allow for ID evaluation, and explain how to construct FSL benchmarks to provide more reliable OOD evaluation. Our hope in highlighting these evaluation concerns is for future researchers to consider them when evaluating newly proposed meta-learning methods or designing new FSL benchmarks.
\vspace{-4pt}
\section{Related Work}
\label{sec:relwork}
\vspace{-6pt}
\textbf{Current FSL benchmarks.} 
A plethora of few-shot image classification benchmarks (\eg \textit{mini}-ImageNet (\mini in short) \cite{vinyals2016matching}, \cifpaper \cite{bertinetto2018meta}) have been developed for FSL evaluation.
These benchmarks typically provide three disjoint sets of classes (base, validation, novel) taken from  standard classification datasets, \eg ImageNet or CIFAR-100~\cite{santoro2016meta,ravi2016optimization,vinyals2016matching}. 
Training, val, and test tasks are then constructed from these classes respectively, which, as we discuss in Section~\ref{sec:background}, \textit{can induce a shift in their corresponding task distributions}. 
Distribution mismatch can be particularly large with non-random splits created at the super class level, \eg \FCpaper \cite{oreshkin2018tadam}, or  dataset level, \eg Meta-Dataset \cite{triantafillou2019meta}. Recently, \citet{arnold2021embedding} propose an automated approach to construct different class splits from the same dataset to allow for varying degrees of task distribution shifts; \citet{triantafillou2021learning} separate the distribution shifts on Meta-Dataset into weak vs. strong generalization depending on whether the novel classes are taken from the used training datasets or not. Both works find that the meta-learning methods that perform better in one distribution shift scenario might perform worse in another, providing further evidence to our OOD performance comparison inconsistency argument in Section~\ref{subsec:inconsistencies-in-OOD}.
Beyond these canonical ways of task construction through a set of classes, \citet{ren2020flexible}  propose  new benchmarks in a new flexible few-shot learning (FFSL) setting, where the aim is to classify examples into a context instead of an object class. During testing, they perform OOD evaluation on unseen contexts. Inspired by their approach, we also provide an FSL benchmark where tasks are specified by contexts (Section~\ref{sec:ID}), though we differ by exploring ID evaluation. 


\textbf{Mismatch between meta-learning theory/methodology and evaluation.} 
Despite theoretical works which analyze meta-learning OOD generalization~\cite{du2020few, fallah2021generalization}, there are many theoretical meta-learning works \cite[e.g.,][]{al2021data,bai2020important,chen2020closer,khodak2019adaptive} that first assume the train and test tasks are \iid sampled from the same distribution despite validating their analyses on OOD FSL benchmarks.
Additionally, several popular meta-learning methods \cite{lee2019meta, finn2018probabilistic, rajeswaran2019meta} that are motivated in the ID scenario are largely evaluated on OOD benchmarks (see Appendix~\ref{sec:quote-meta-learning-methods}). 
Prior work of \citet{lee2019learning} explores ID vs. OOD; however they treat the FSL setup as ID when the disjoint base and novel classes are from the same dataset and OOD 
only when they are not (\eg base from ImageNet, novel from CUB). 
We emphasize that even the disjointedness of base and novel from the same dataset can create a task distribution shift and hence unlike \cite{lee2019learning} we do not consider current FSL benchmarks (like \mini) to be ID (Section~\ref{sec:background}). 

\textbf{FSL ID evaluation.} 
We are unaware of any FSL classification benchmarks that are explicitly advertised for ID evaluation. As we discuss in Section~\ref{sec:ID}, a growing number of works~\cite{chen2018federated,fallah2020personalized,jiang2019improving,khodak2019adaptive,lin2020collab} use meta-learning for personalized federated learning, but do not clearly discuss the in-distribution nature of these benchmarks nor how they differ from standard FSL benchmarks. In recent work, \citet{chen2020new}
extend current FSL benchmarks to evaluate their proposed method both ID and OOD, but only to further improve OOD performance on the original FSL benchmark's novel classes. Our work uses a similar setup for constructing ID evaluations with current OOD FSL benchmarks, but focuses on showing that certain meta-learning methods/design choices can improve OOD performance at the cost of ID performance.
Prior work \cite{gidaris2018dynamic, ren2019incremental} on incremental few-shot/low-shot learning also explores performance on both base and novel classes simultaneously. However, they differ in their methodology, as they use supervised learning (not meta-learning) to classify over the base classes. 


\textbf{OOD evaluation in other fields.} Train and test distribution shifts are 
also found in domain generalization \cite{blanchard2011generalizing,muandet2013domain} where the goal is to find a model that works well for a different test environment. Due to this shift, 
\citet{gulrajani2020search} specifically discuss difficulties in performing model selection in domain generalization benchmarks. 
They argue that it is the responsibility of the  method designer (not benchmark designer) to determine model selection strategies for their method, and propose several model selection methods, mainly targeting hyperparameter selection. Unlike domain generalization, FSL benchmarks often have a pre-determined disjoint set of validation classes to construct validation tasks, so the need for developing model selection methods may be less obvious. In Section~\ref{sec:OOD}, motivated by \cite{gulrajani2020search}, we explore strategies for model selection for meta-learning.
However, in contrast to~\cite{gulrajani2020search}, we focus on algorithm snapshot selection for meta-learning, which is required by hyperparameter selection as a subroutine (exact definitions see Section~\ref{sec:OOD}).

\vspace{-6pt}
\section{FSL benchmarks: Background \& Focus on OOD evaluation}
\label{sec:background}
\vspace{-4pt}
\textbf{{Background and notation.}} In this work, we employ a general definition of a meta-learning FSL task: a task $\calT$ is a distribution over the space of support and query dataset pairs $(S, Q)$, where $S$, $Q$ are two sets of examples from an example space $\calX \times \calY$. The support set $S$ is used by an algorithm to produce an adapted model which is then evaluated by the corresponding query set $Q$. Each time we interact with a task $\calT$, an $(S, Q)$ pair is sampled \iid from $\calT$, the mechanism for which depends on the specific application; we provide multiple examples below.  As discussed in Section~\ref{sec:introduction}, meta-learning aims to learn an algorithm over a distribution of tasks $\Prob(\calT)$. 
During meta-training, we assume access to $N$ pairs of $\{(S_i, Q_i)\}_{i \in [N]}$ sampled from a \textit{training task distribution\footnote{We note that ``training task distribution'' here refers to the true underlying distribution of training tasks, not the empirical distribution supported over the finite set of sampled training tasks.}} $\Prob_\tr(\calT)$ in the following way: first, $N$ tasks are sampled \iid from the training task $\calT_i \sim \Prob_\tr(\calT)$; then, for each task $\calT_i$, a support query pair $(S_i, Q_i) \sim \calT_i$ is sampled \iid.
During meta-testing, we assume there is a \textit{test task distribution} $\Prob_\te(\calT)$ where fresh $(S, Q)$ samples are similarly sampled based on $P_\te(\calT)$. We define a learning scenario to be in-distribution (ID) if $\Prob_\tr$ $=$ $\Prob_\te$ and out-of-distribution (OOD) if $\Prob_\tr$ $\neq$ $\Prob_\te$. Whenever $\Prob_\tr =\Prob_\te$, the induced train and test marginal distributions of $(S, Q)$ are also identical. 

\textbf{{Construction of $(S,Q)$ pairs in FSL.}} Most popular FSL benchmarks share a similar structure: they provide three disjoint sets of classes: base classes $\calC_B$, validation classes $\calC_V$, and novel classes $\calC_N$, where any class $c$ in these sets specifies a distribution $\Prob_c$ over the example space $\calX$. An $n$-way $k$-shot $q$-query task $\calT_{\bc}$ in these benchmarks is specified by a length $n$ non-repeated class tuple $\bc = (c_1, \ldots, c_n)$ where $\bc \in [\calC^n] \coloneqq$ $\{(d_1, \ldots, d_n)$ $\in$ $\calC^n: d_i \neq d_j, \forall i \neq j\}$. $\calT_{\bc}$ generates random $(S, Q)$ pairs in the following way: For every class $c_i$, $k$ support examples $S_i \sim (\Prob_{c_i})^k$ and $q$ query examples $Q_i \sim (\Prob_{c_i})^q$ are sampled. The support and query set is formed by the union of such labelled examples from each class: $S = \cup_{i=1}^n \{(x, i), \forall x \in S_i\}$, $Q = \cup_{i=1}^n \{(x, i), \forall x \in Q_i\}$.
By specifying the base and novel classes, the FSL benchmark has provided a collection of tasks for training $\{\calT_{\bc}: \bc \in [\calC_B^n]\}$ and test $\{\calT_{\bc}: \bc \in [\calC_N^n]\}$. These sets can be extremely large. For example, in \mini, which has $64$ base classes, the total number of $5$-way training tasks is {\small $\frac{64!}{(64-5)!} \approx 9.1  \times 10^8$}. However, it is not explicitly specified what underlying task distribution $\Prob(\calT)$  generates these sets of tasks. We believe this may have led prior work to discuss FSL benchmarks in the context of ID evaluation~\cite[e.g.,][]{al2021data,bai2020important,chen2020closer}, contrary to what we outline below.

\textbf{{Current FSL benchmarks target OOD evaluation.}} We now informally discuss our reasoning for arguing that current FSL benchmarks reflect OOD evaluation (we provide a more formal proof by contradiction in Appendix~\ref{app:formal-analysis}). 
In particular, if the training and test tasks in FSL benchmarks are indeed \iid sampled from the same underlying task distribution, then this underlying distribution must be induced by a distribution over class tuples of an even larger class set $\calC_L$ ($\calC_L \supseteq (\calC_B \cup \calC_V \cup \calC_N$)).

We consider the following dichotomy: 
\vspace{-0.5em}
\begin{enumerate}[label=\roman*),topsep=0pt,itemsep=-1ex,partopsep=1ex,parsep=1ex, leftmargin=15pt]
    \item when $|\calC_L| = \calO(nN)$: In this case, the total number of classes $nN$ covered by the sampled tasks (counting repetition) is greater than the number of underlying classes. Then with high probability, both the training and test tasks would each cover a significant portion of all the classes in $\calC_L$, making it extremely unlikely to have an empty intersection as in the current FSL benchmarks. 
    \item when $|\calC_L| = \Omega(nN)$: In this alternative case, the total number of classes (even counting repetitions) used by sampled tasks is still smaller than the number of underlying classes $|\calC_L|$. Thus the sampled tasks cannot cover all the underlying classes. Under this regime, the number of classes covered by the training tasks alone would scale linearly with ${N}$, as repeating an already-seen class in a new task sample is relatively rare. Since FSL benchmarks typically use a large number of training tasks during meta-training ($N>10^3$), it is improbable that all the training tasks would together only cover a very low number of classes ($64$ in the case of \mini).
\end{enumerate}

\textbf{Randomized class partitions do not imply randomized task partitions.} Another issue that may cause some to view the current FSL benchmarks as performing ID evaluation is that in some of these benchmarks, the base, val, novel classes are random partitions of \iid drawn classes from a class level distribution (specifically \minipaper, \cifpaper; but not \FCpaper, \tieredpaper as the classes are not partitioned randomly). The logic here is that in standard machine learning practice, randomly partitioning \iid sampled data points into train and test guarantees that the train and test samples are drawn \iid from the same underlying distribution. 
However, it is important to notice that \textit{the first class citizen in common FSL benchmarks is not a class, but a task} (represented by a class tuple). So, only a randomized partition of \iid sampled class tuples would guarantee in-distribution sampling. 

\textbf{How can we view $\Prob_{\tr}, \Prob_{\te}$ in common FSL benchmarks?} Based on the discussion above, we need to view train and test tasks in current FSL benchmarks as coming from different distributions, \ie $\Prob_{\tr} \neq \Prob_{\te}$.
In order to ensure that both sets of tasks are still sampled \iid from their respective distributions, it is convenient to view the train/test tasks as being \iid sampled from a uniform distribution over all possible train/test class tuples induced by $\calC_B/\calC_N$ \ie $\Prob_{\tr} = \Prob_{\calC_B} \coloneqq \Unif(\{\calT_{\bc}: \bc \in [\calC_B^n]\})$ and test $\Prob_{\te} = \Prob_{\calC_N} \coloneqq \Unif(\{\calT_{\bc}: \bc \in [\calC_N^n]\})$ --- a view which we will adopt in the rest of the paper.

\vspace{-12pt}
\section{Evaluating In-Distribution Performance}
\vspace{-12pt}
\label{sec:ID}

Although (as discussed in Section~\ref{sec:background}) current FSL benchmarks target OOD evaluation, we now explore example applications where ID generalization is instead required, and provide easily constructible benchmarks mirroring these scenarios. As we will show, this distinction is important because meta-learning methods may perform markedly different in ID vs OOD scenarios.

\newcommand{\id}{\texttt{id}}
\textbf{\textit{Example 1} (Federated Learning):}
Multiple works \cite{chen2018federated,fallah2020personalized,jiang2019improving,khodak2019adaptive,lin2020collab} have considered applying meta-learning methods in federated learning, in which the goal is to learn across a distributed network of devices~\cite{mcmahan2021advances,li2020federated}. Meta-learning can produce personalized models for unseen devices/users, improving over a single globally-learned model's performance. 
In this setting, a popular benchmark is the FEMNIST \cite{caldas2018leaf} handwriting recognition dataset.
For FEMNIST, we assume there exists a distribution of writers $\Prob(\id)$ in a federated network where each writer (with a unique \id) is associated with a few-shot classification problem to recognize over the different character classes the writer has written for.
We associate each writer $\id$ with a task $\calT_{\id}$ which randomly generates a support set with one random example per class and a query set with varying number of random examples per class.

\textbf{ID evaluation on FEMNIST.} We are given a total of $\sim3500$ writers sampled \iid from $\Prob(\id)$ and we randomly partition them into a $2509/538/538$ split for training, validation, and test tasks, following similar practices used in prior FL work~\cite{jiang2019improving, chen2018federated}. Note that this random split is performed at the task/\id ~ level. As such, we can treat the training and test tasks as being sampled \iid from the same task distribution, unlike current FSL benchmarks.




\textbf{\textit{Example 2} (Online Recommendation):}
\citet{ren2020flexible} propose the use of \textit{Zappos} \cite{yu2014fine} dataset as a meta-learning benchmark where each task is a binary classification of shoe images into an attribute context. This mimics an online shopping recommendation problem, where each user has different shoe preferences based on specific shoe attributes (hence a single global predictive model would not perform well), and the recommendation system 
must quickly learn a user's likes/dislikes through a few interactions.
In this simplified setup, we fix a universal set of shoe attributes $\calA$, and each user's preference is represented by a specific pair of unique attributes $\ba = (a_1, a_2),\, a_1 \neq a_2$. A task $\calT_{\ba}$ representing a user with attribute preference $\ba$ generates $2$-way $k$-shot $q$-query $(S, Q)$ pair by \iid sampling $k+q$ examples both from the set of images that carry both attributes in $\ba$ (positive examples) and from the set that does not (negative examples).
Our task distribution is a uniform distribution over tasks of all attribute pairs $\Prob_{\calA}(\calT) = \Unif(\{\calT_{\ba}: \ba \in [\calA^2])\}$. 

\textbf{ID evaluation on Zappos.} Unlike in \cite{ren2020flexible}, where Zappos is used to measure OOD performance by having disjoint train and test attributes $\calA_\tr \cap \calA_\te = \phi$, in this work we use Zappos for ID evaluations by \iid sampling both meta-train and meta-test tasks from the same distribution $\Prob_{\calA}(\calT)$.
Through this modification of \citeauthor{ren2020flexible}'s setup, we sample $1000$ / $50$ attribute pairs from an attribute set $|\calA|=36$ to construct $1000$ / $50$ training tasks (each with a randomly sampled $(S, Q)$) and evaluate ID performance on another $25000$ test tasks sampled in the same way. 
Our evaluation setup captures a realistic setting where the goal is to learn an algorithm that can generalize to the entire online shopper population despite being trained only on a randomly chosen subset of shoppers.




\textbf{Remark.} Even with ID evaluation it is possible to encounter unseen classes/attributes in meta-test tasks, specifically when the number of meta-train tasks is smaller than the number of underlying classes/attributes (Section~\ref{sec:background}). 
However, a suitable number of \iid sampled meta-train tasks is needed to ensure good performance, which would naturally encompass a larger set of meta-train classes than those considered by OOD FSL benchmarks.
For example, there are still $16$ attribute pairs from the test tasks that are unseen in the $1000$ training tasks on Zappos-ID, but the dataset is still in-distribution since the sampling distributions of train and test attribute pairs (and thus of tasks) are identical.

    
        
        

\begin{table}[t]
    \scriptsize
    \centering
    \caption{Ranking in () of meta-algorithms' test performance on \textbf{i)} ID benchmarks FEMNIST, Zappos-ID (with either 1000 or 50 training tasks); and \textbf{ii)} OOD FSL benchmark \minipaper.} 
    \vspace{0.25em}
    \setlength{\tabcolsep}{4.7pt}
    \label{tab:results-id}
    \begin{tabular}{|c|c|c|c|c||c|}\hline
    {Dataset /} & {FEMNIST} & \multicolumn{3}{c||}{ Zappos-ID} & \minipaper \\ \cline{2-5}
    {Method}  & *w1s & 2w10s & 2w5s (1000 train tasks) & 2w5s (50 train tasks) & 5w5s \\ 
    \Xhline{2\arrayrulewidth}
    {PN} 
        &  $^\mathbf{(1)}94.72 \pm 0.41 \%$  
        &  $^\mathbf{(1)}88.40 \pm 0.13 \%$             
        &  $^\mathbf{(1)}86.58 \pm 0.15 \%$   
        &  $^\mathbf{(1)}77.67 \pm 0.17 \%$   
        &  $^\mathbf{(3)}76.22 \pm 0.14 \%$  \\
    
    {Ridge} 
        &  $^\mathbf{(1)}94.71 \pm 0.42 \%$  
        &  $^\mathbf{(2)}88.01 \pm 0.14 \%$             
        &  $^\mathbf{(2)}85.56 \pm 0.16 \%$     
        &  $^\mathbf{(2)}74.75 \pm 0.16 \%$     
        &  $^\mathbf{(2)}77.20 \pm 0.15 \%$  \\
        
    {SVM} 
        &  $^\mathbf{(3)}94.22  \pm 0.45 \%$  
        &  $^\mathbf{(3)}87.75  \pm 0.14 \%$             
        &  $^\mathbf{(3)}85.12  \pm 0.16 \%$   
        &  $^\mathbf{(3)}74.06  \pm 0.17 \%$   
        &  $^\mathbf{(1)}77.72  \pm 0.15 \%$  \\
        
    FO-MAML 
        & N/A 
        &  $^\mathbf{(4)}81.90  \pm 0.14 \%$             
        &  $^\mathbf{(4)}80.14  \pm 0.15 \%$   
        &  $^\mathbf{(4)}69.85  \pm 0.18 \%$   
        &  $^\mathbf{(4)}75.96  \pm 0.17 \%$  \\ \hline
    \end{tabular}
    \vspace{-3em}
\end{table}

\textbf{ID benchmark results.} We evaluate the ID performance of four popular meta-learning methods: Prototypical Networks (\pn) \cite{snell2017prototypical}, MetaOptNet-SVM (\svm) \cite{lee2019meta}, MetaOptNet-Ridge Regression (\ridge) \cite{lee2019meta, bertinetto2018meta} and FOMAML \cite{finn2017model} on our identified ID FSL benchmarks (Table~\ref{tab:results-id}). 
Since FEMNIST's tasks have varying number of ways, FOMAML cannot be directly used due to the logit layer shape mismatch.
We note that the performance order of the four methods are consistent on all three ID benchmarks yet surprisingly \textbf{almost completely opposite} to the performance order observed on the OOD benchmark \mini (except for FOMAML). In terms of the actual performance differences, we notice that the ID performance advantage of \pn over \svm becomes particularly large ($>3\%$) when we reduce of the number of training tasks for Zappos-ID to 50; in contrast, on the OOD benchmark \mini, \svm instead outperforms \pn by 1.5\% (a significant number as many newly proposed meta-learning methods often only report improvements over previous methods by $1.5\%$ or less). These performance differences make it clear that the performance ranking flips between ID and OOD indeed exist, and as a result, these common OOD FSL benchmarks (like \mini) cannot be used to compare ID performance without modifications, giving further evidence to the danger of such practices (see Section~\ref{sec:relwork}). To further understand this phenomenon, we propose a way to also enable ID performance evaluation over these common OOD FSL benchmarks and see if there still exists a difference in ID, OOD performance orders.




\textbf{Modifying OOD FSL benchmarks for ID evaluation.} From our previous discussion, we have shown that we can think of FSL training tasks as being sampled \iid from the task distribution $\Prob_{\calC_B} \coloneqq \Unif(\{\calT_{\bc}: \bc \in [\calC_B^n]\})$. To conduct an in-distribution evaluation, 
we need to sample fresh test tasks \iid from $\Prob_{\calC_B}$. 
For a freshly sampled $\calT_{\bc} \sim \Prob_{\calC_B}$, we also need to  \iid  sample a fresh support query pair $(S, Q) \sim \calT_{\bc}$. To ensure independent sampling from the already seen meta-training  $(S, Q)$ pairs, we need to introduce new examples from $\Prob_c$ for each class $c \in \calC_B$.
Thus we construct slightly modified versions of four common FSL benchmarks \textbf{i)} {\minipapermod (\minimod)} \cite{vinyals2016matching}, in which we find ($\approx700$) unused examples (from ImageNet) for each base class and use them to evaluate the performance over $\Prob_{\calC_B}$. Here the original 600 examples of each base class are still only used for meta-training.
\textbf{ii)} {\cifpapermod (\cifmod)} \cite{oreshkin2018tadam}, {\FCpapermod (\FCmod)}  \cite{bertinetto2018meta}, and  {\tieredpapermod (\tieredmod)} \cite{ren2018meta}: As we don't have additional samples for base classes, we randomly partition each base class's current examples into an approximate $80/20$ split where the training tasks are constructed using the former and the latter is reserved for ID evaluation.
\begin{figure*}[t]
  \hfill
  \begin{minipage}{.22\textwidth}
    \includegraphics[width=\linewidth]{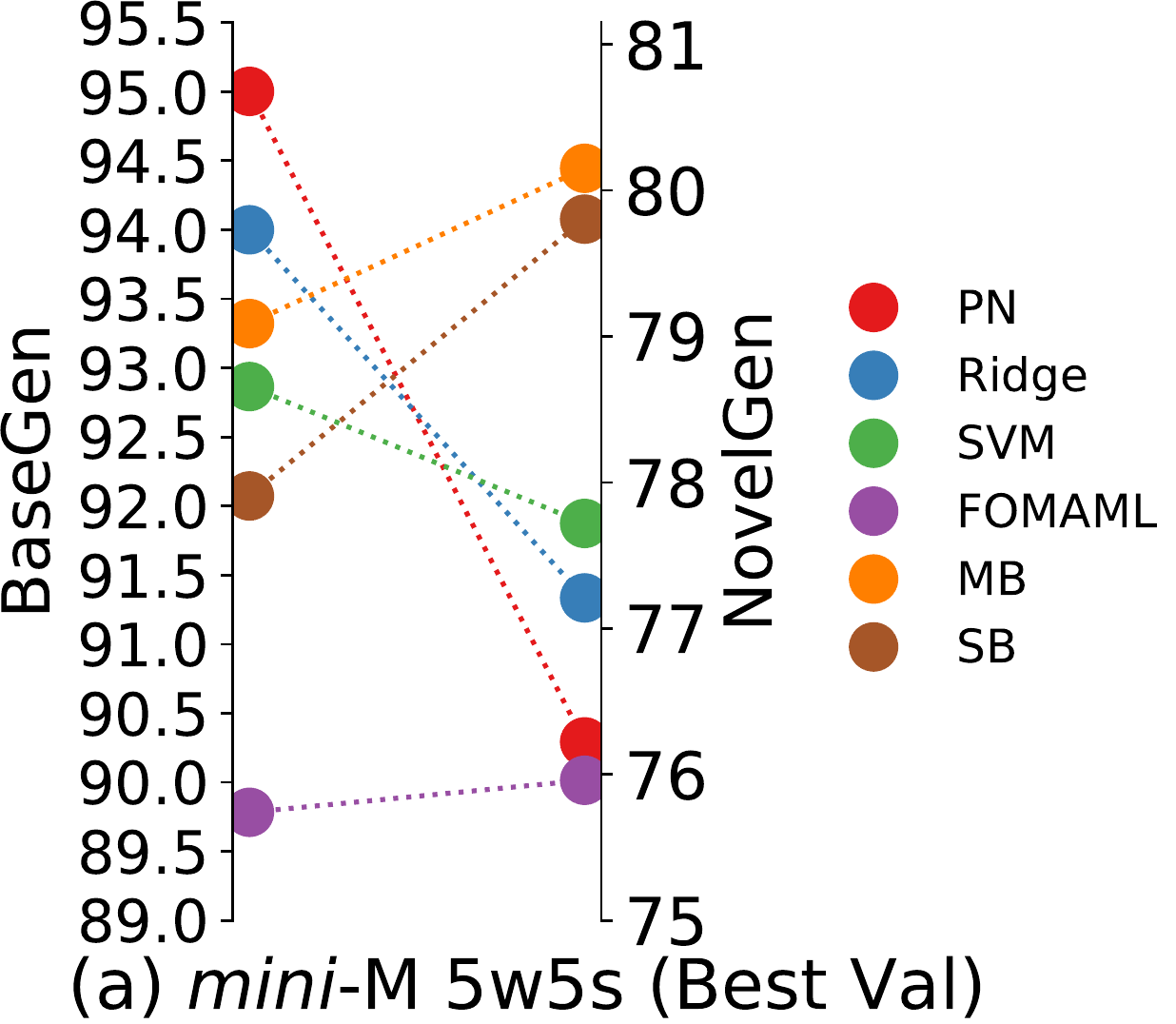}
  \end{minipage}
  \begin{minipage}{.225\textwidth}
    \includegraphics[width=\linewidth]{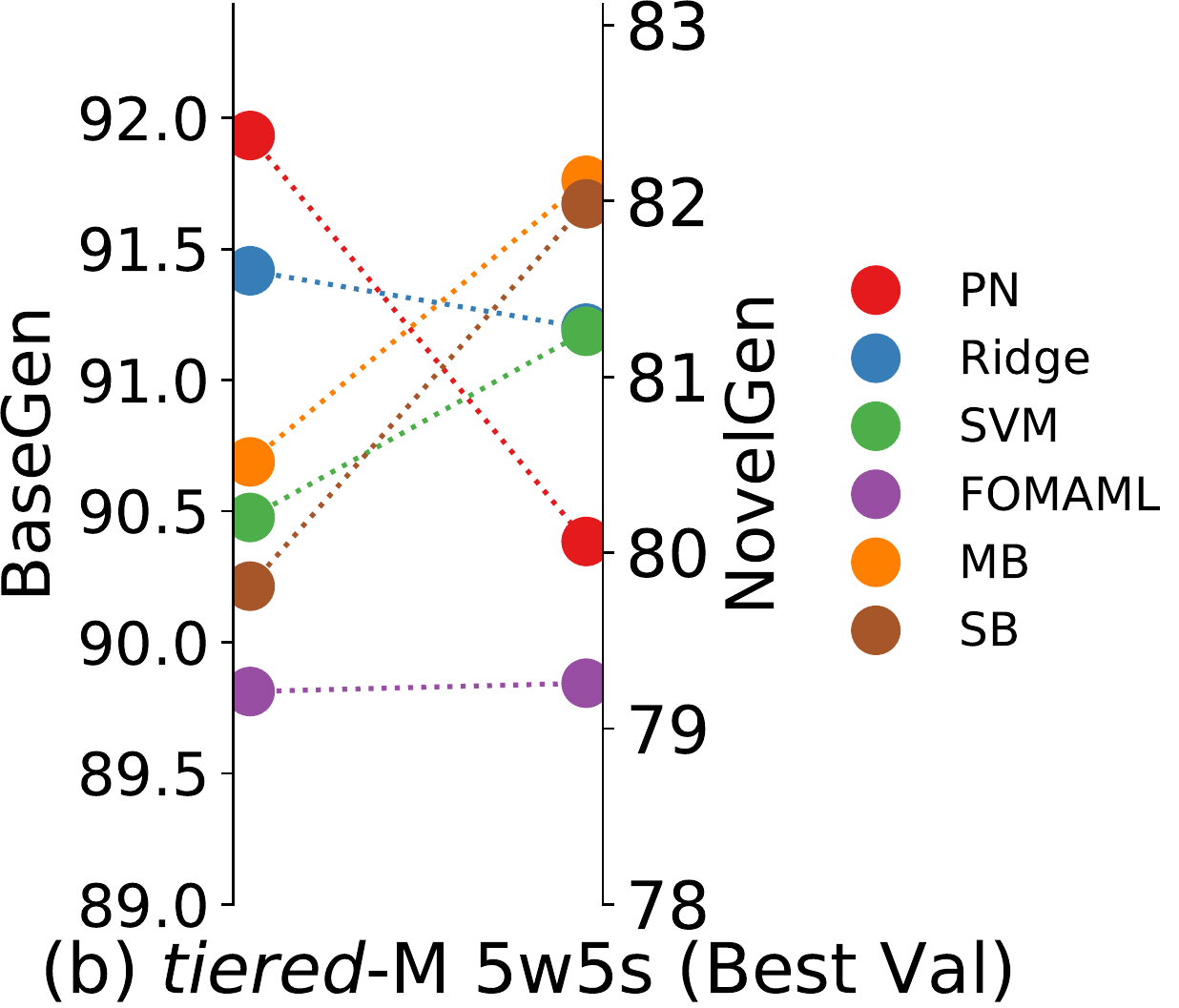}
  \end{minipage}
      \begin{minipage}{.265\textwidth}
    \includegraphics[width=\linewidth]{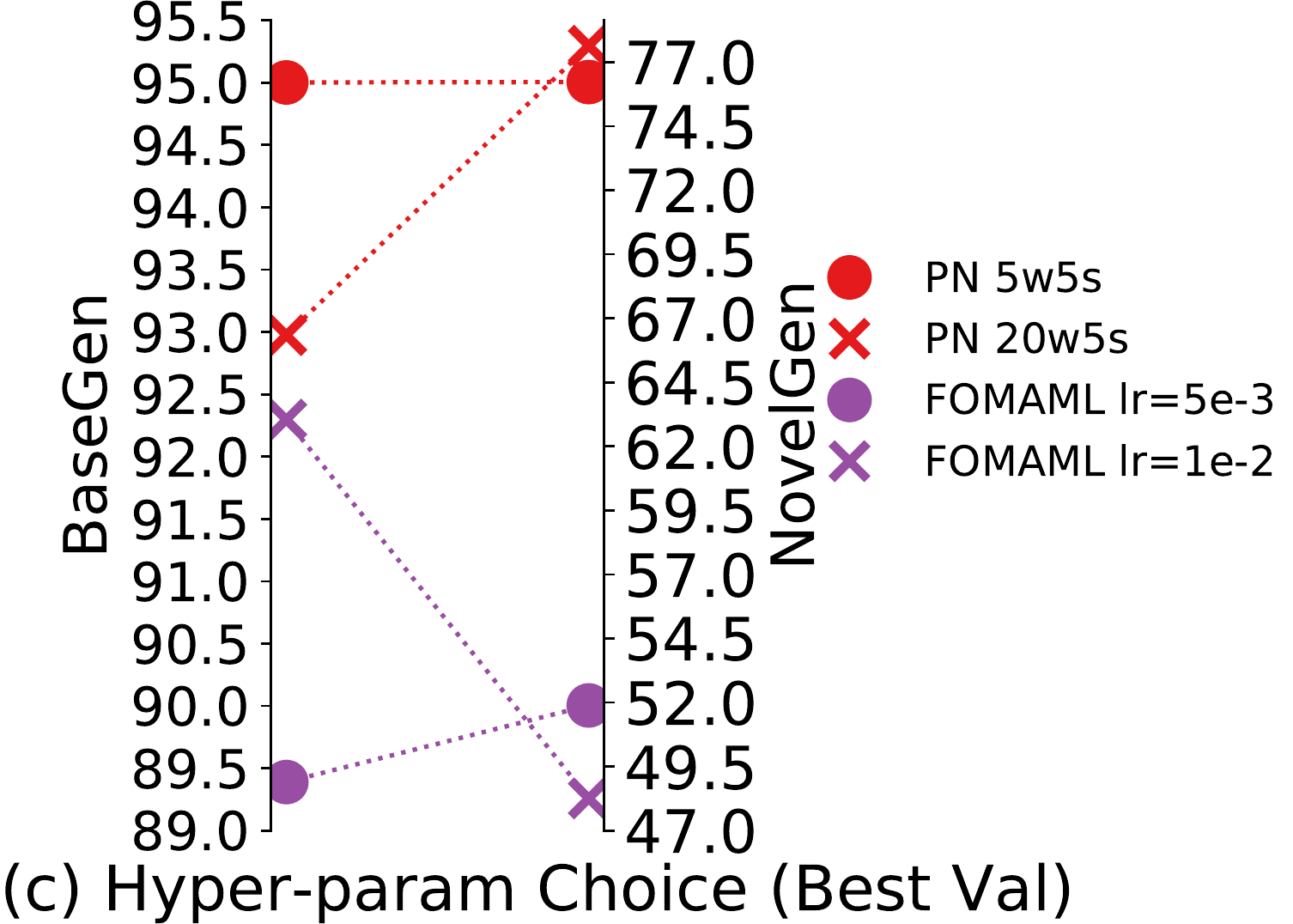}
  \end{minipage}
  \begin{minipage}{.24\textwidth}
    \includegraphics[width=\linewidth]{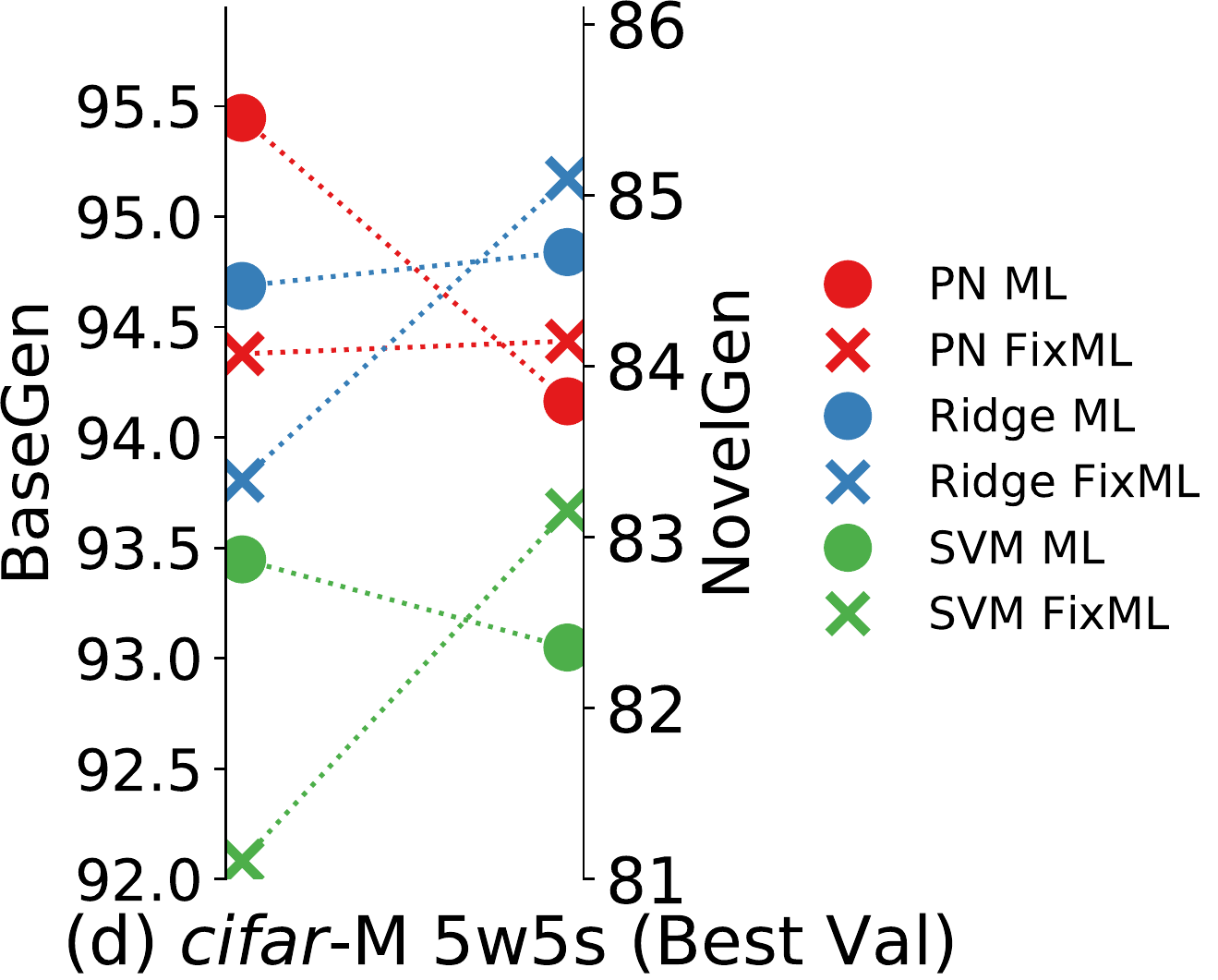} 
  \end{minipage} 
 \caption{We show the BaseGen and NovelGen performance tradeoff (for best validation snapshots): over the choice of a set of four meta-learning and two supervised pre-training methods on \minimod (a) and \tieredmod (b); over the number of ways to train \pn on \minimod and different learning
  rates to train \fomaml on \FCmod (c); over the use of \fixml $(S, Q)$ 
  generation strategy or not (\ml) with \svm, \ridge and \pn on \cifmod in (d). }
 \label{fig:id-ood-tradeoffs}
\end{figure*}

\textbf{ID vs. OOD conflicts still exist.} In addition to evaluating the four aforementioned meta-learning methods, we also consider two supervised pretraining methods: the supervised learning baseline (SB) \cite{chen2019closer} and Meta-Baseline (MB) \cite{chen2020new}.
Both of these methods have been shown to outperform many meta-learning methods
on the current OOD FSL benchmarks. 
For succintness, we call the generalization performance over the training task distribution $\Prob_{\calC_B}$ \textit{BaseGen} and performance over $\Prob_{\calC_N}$ \textit{NovelGen} (we will refer to performance on validation tasks from $\Prob_{\calC_V}$ \textit{ValGen} in a later section). 
We plot the BaseGen and NovelGen performances of the aforementioned methods on \minipapermod and \tieredpapermod in Figure~\ref{fig:id-ood-tradeoffs}(a)(b) (other datasets see Appendix~\ref{app:additional-id} Figure~\ref{fig:additional-id-plots}(a)(b)), with a dotted line connecting BaseGen and NovelGen value of same learned algorithm snapshot of a meta-learning method. 
We see that the \textbf{ID/OOD performance order flips still exist within current FSL benchmarks themselves} (the dotted lines of different methods cross each other frequently), showing that \textbf{the issue of improving OOD at the expense of ID is a common realistic phenomenon for multiple benchmarks}.
In addition, despite outperforming all meta-learning methods on NovelGen, the non-meta-learning methods SB and MB cannot beat the best meta-learning methods on BaseGen, which demonstrates their restrictive advantage only in the OOD setting.
More broadly speaking, these OOD FSL datasets are constructed with a belief that there exists a commonality between the training and test task distributions so that an algorithm capturing this commonality using the training distribution alone would also generalize to the test tasks. Thus the phenomenon of improving OOD while sacrificing ID means the algorithm has in some sense failed to learn the essence/commonality from the given training tasks.
Additionally, we see that the \textit{BaseGen ranking of the four meta-learning methods over these common FSL bechmarks are exactly the same as the ranking over our two newly proposed ID benchmarks in Table~\ref{tab:results-id}}. We suggest that researchers who want to also test on their proposed methods’ ID generalization performance can perform BaseGen evaluation method in our proposed way as it is a simple addition to their existing NovelGen evaluation setup.

\textbf{Certain OOD training choices might harm ID generalization.} In addition to checking the discrepancy of ID vs OOD generalization comparison of different meta-learning methods, we now ask, for a given meta-learning method, whether the meta-training choices found to be most effective on NovelGen would still be optimal for BaseGen. 
\begin{enumerate}[label=\roman*),topsep=-0.5ex,itemsep=-1ex,partopsep=0ex,parsep=1ex,leftmargin=15pt]
    \item \textbf{Meta-training hyperparameters:} In Figure~\ref{fig:id-ood-tradeoffs}(c) we see that a higher learning rate when training \fomaml finds algorithms that have higher BaseGen, whereas lower learning rates are better for NovelGen. Additionally, we see that the proposed technique of training with more ways for \pn in \cite{snell2017prototypical} can lead to better NovelGen performance but worse BaseGen performance than the $5$-way trained \pn whose training and test task configurations match.
    \item \textbf{Meta-training $(S, Q)$ generation alternatives:} It was found in \cite{setlur2020support} that always using the same support examples for every base class when constructing $S$ (\fixml) can improve NovelGen performance for several meta-learning methods over multiple OOD FSL benchmarks. However, restricting the training $(S, Q)$'s diversity sampled from $\Prob_{\calC_B}$ seems counter-intuitive and we wouldn't expect to improve the in-distribution generalization BaseGen. In Figure~\ref{fig:id-ood-tradeoffs}(d), we indeed see that \fixml only improves NovelGen performance at the expense of BaseGen by training over a much less diverse set of tasks.
\end{enumerate}
These two observations above caution us that \textbf{some training techniques to boost test accuracy on FSL benchmarks might only work for the OOD scenario but not the ID scenario.}

\vspace{-4pt}
\section{Challenges With Out-of-Distribution Evaluation}
\vspace{-6pt}
\label{sec:OOD}

After identifying some example benchmarks for ID evaluation, we come back to the current OOD FSL benchmarks to further examine some subtle challenges in OOD evaluation. Here, instead of focusing on the distinction between ID vs. OOD, we now look at some reliability and inconsistency problems \textit{within the OOD evaluation itself}. In particular, we highlight two concerns: \textbf{1)} Despite providing a validation set of classes for task construction during meta-validation, it is not clear whether this is a reliable way to perform model selection. \textbf{2)} As there is no requirement on how similar the training and test task distributions need to be in the OOD setting, there can be inconsistencies in method comparisons when the evaluation setup is slightly modified.

\vspace{-0.5em}
\subsection{Model Selection}
\vspace{-0.5em}
\label{subsec:model-selection}
To compare a set of meta-learning methods on the OOD test task distribution, one needs to \textit{select} a representative algorithm learned by each method to be evaluated on the test task distribution.
To do so, one should first decide what  set of hyperparameter configurations to choose from for a given meta-learning method (we define the entire set of all hyperparameters for a training run as a \textit{hyperparameter config}). For each such considered config, after its training is completed, we need to choose one of the algorithm snapshots saved during training to represent it (which we call \textit{snapshot selection}). Then the set of hyperparameter configs are compared based on their respectively selected snapshots, and a single config is then chosen among them (which we term \textit{hyperparameter selection}). This config's selected snapshot will represent the given meta-learning method to be finally evaluated on the test tasks. We refer to the combined problem of hyperparameter and snapshot selection as model selection. (See Appendix~\ref{app:subsec:snapshot-hyperparameter-selection-example} for a simplified example of snapshot and hyperparameter selection; see \ref{app:subsec:early-stopping} for the distinction between the commonly used technique \textit{early-stopping} and snapshot selection.)

\textbf{Snapshot vs Hyperparameter selection.} If the snapshot selection strategy cannot reliably identify a snapshot with good test performance, it is possible that some hyperparameter configs will be unfairly represented by a mischosen snapshot, leading to erroneous hyperparameter selection.
Thus we believe snapshot selection is a more fundamental problem and we focus our analyses on it (more on hyperparameter selection in Appendix~\ref{app:subsec:additional-ood-hp}).
In OOD FSL benchmarks, a reserved set of validation classes (disjoint from the set of novel classes) is provided, which as we have argued, provides tasks which are not sampled \iid from the test task distribution. As a result, we ask: \textit{is performing validation on the given validation tasks reliable, and, if not, are there other options?}
In contrast, there is typically no need for such concerns in standard ID supervised learning, where the validation set is sampled \iid from the test distribution.

\begin{figure*}[t]
 \hfill
  \begin{minipage}{.215\textwidth}
  \setlength{\abovecaptionskip}{2pt} 
    \includegraphics[width=\linewidth]{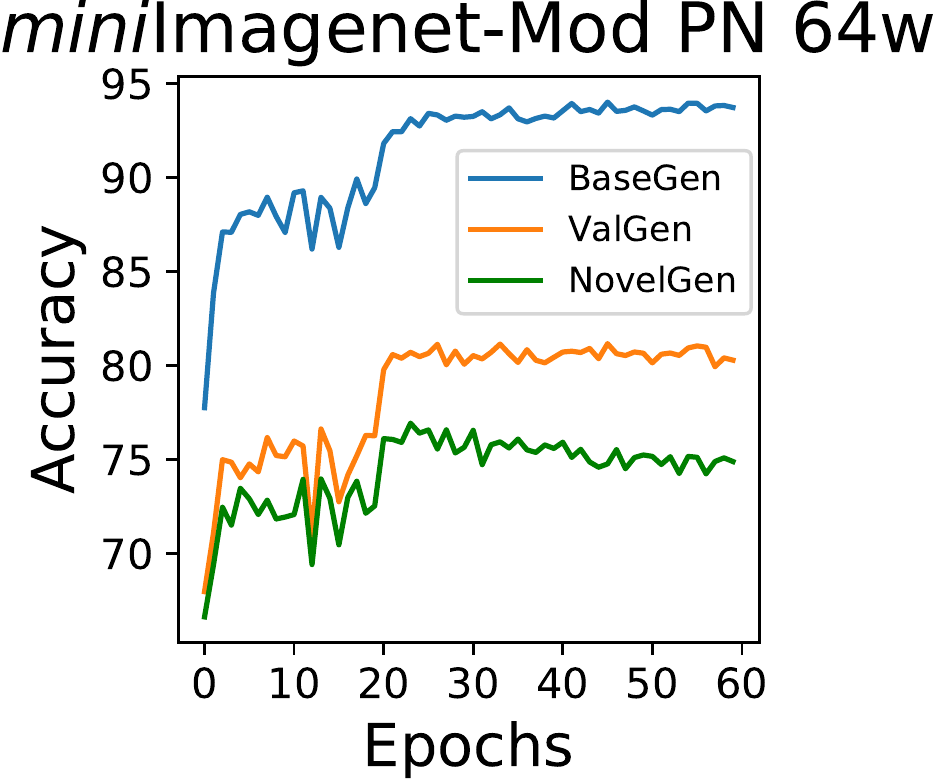}
    \caption*{\scriptsize (a)}
  \end{minipage}
    \begin{minipage}{0.30\textwidth}
    \setlength{\abovecaptionskip}{2pt} 
    \includegraphics[width=\linewidth]{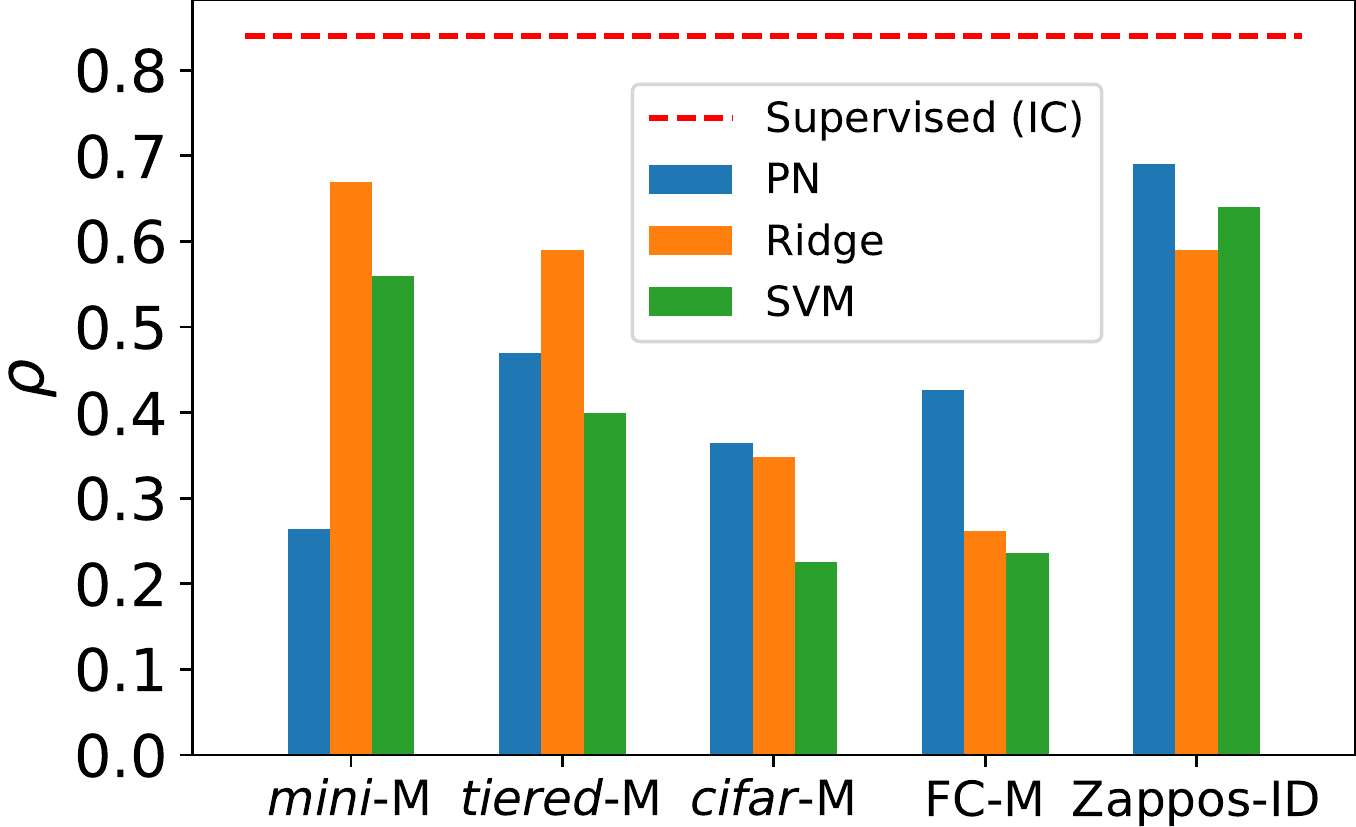}
    \caption*{\scriptsize $\quad\,\,$(b)}
    \end{minipage}
  \begin{minipage}{.24\textwidth}
  \setlength{\abovecaptionskip}{2pt} 
    \includegraphics[width=\linewidth]{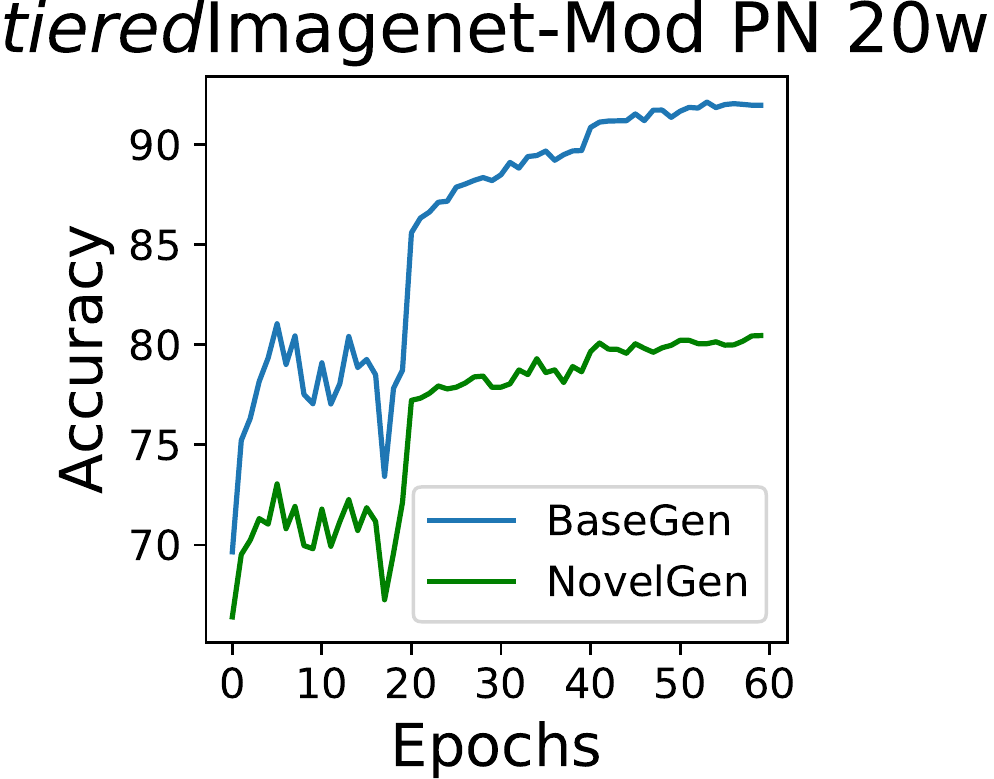}
    \caption*{\scriptsize (c)}
  \end{minipage}
  \begin{minipage}{.19\textwidth}
  \setlength{\abovecaptionskip}{2pt} 
    \includegraphics[width=\linewidth]{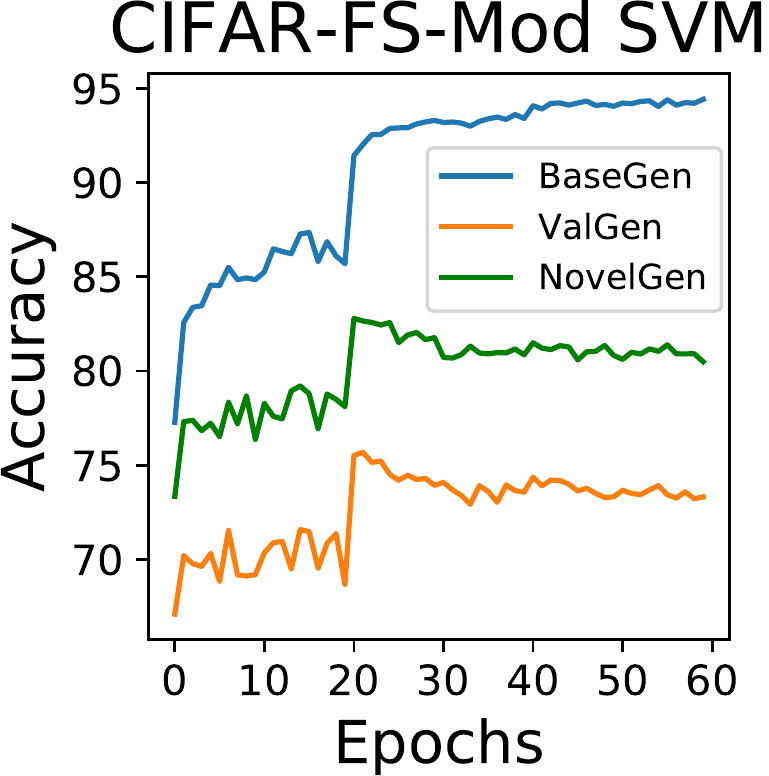}
    \caption*{\scriptsize$\quad\,$(d)}
  \end{minipage} \hfill
  \vspace{0.75em}
 \caption{We plot the (Base, Val, Novel)Gen progression of 64(max)-way trained \pn on \minimod in (a) and of \svm trained on \cifmod in (d). In (b) we compute the Kendall rank correlation coefficient ($\rho$) between the validation and test rankings of model snapshots for IC (trained on \cif) and algorithm snapshots for \pn, \svm, \ridge on OOD datasets \minimod (last 40 epochs), \cifmod (last 40), \FCmod (last 10), \tieredmod (last 20) and ID dataset Zappos-ID (last 30); in (c) we show the BaseGen tracking NovelGen for 20-way trained \pn on \tieredmod.}
 \label{fig:model-selection}
\end{figure*}

\subsubsection{\textit{Option 1}: Snapshot selection tradition using ValGen.} 
\label{subsubsec:snapshot-selection-valgen}

By providing a set of validation classes, it has become the default practice for meta-learning methods to use ValGen performance for snapshot selection. However, because $\Prob_{\calC_V}$ and $\Prob_{\calC_N}$ are different tasks distributions, it is not clear whether a higher ValGen performance is \textit{strongly correlated with} a higher NovelGen performance. In Figure~\ref{fig:model-selection}(a), we plot the progression of ValGen and NovelGen of a 64-way trained \pn on \minipapermod 5w5s tasks. We notice that ValGen is consistently higher than NovelGen, indicating that meta-val performance is not an accurate estimator of NovelGen.
More importantly, we see trendwise that while ValGen is generally non-decreasing, NovelGen starts to decrease after epoch 30. Thus the snapshot selected according to the best ValGen is not the snapshot with the best possible meta-test performance. In fact, this loss of NovelGen performance due to choosing the best ValGen model instead of the actual best NovelGen model can be particularly large, with values being 1.1\% for \svm, 1.2\% for \ridge, 0.7\% for \pn, and 0.9\% for \fomaml on the \FCpaper dataset. These performances losses for each method are especially concerning considering the differences among the best possible NovelGen performance of these different methods are often smaller than $1.5\%$.

\textbf{Ranking similarity analysis.} In light of the above observation, we ask a more general quantitative question: \textit{How \textbf{similar} is the ranking of the training snapshots using meta-val performance (ValGen) \textbf{to} the ranking using the meta-test performance (NovelGen)?}
To answer this, we compute the Kendall rank correlation coefficient\footnote{$\rho \in [-1,1]$, $\rho \approx 0$ means no correlation, while $\rho=1/\rho=-1$ means exactly same/opposite rankings. 
} 
$\rho$ \cite{kendall1948rank} between the ValGen and NovelGen rankings of algorithm snapshots trained on four OOD FSL benchmarks (Figure \ref{fig:model-selection}(b)) and our ID benchmark Zappos-ID whose validation and test tasks come from the same distribution $\Prob_{\calA}$.
More concretely, for each meta-learning method and dataset combination, we save the algorithm snapshots (one from each epoch) throughout meta-training and rank these algorithm snapshots according to their ValGen and NovelGen value respectively. Then $\rho$ is computed between these two rankings for this specific (meta-learning method,dataset) combination.
For snapshot selection through ValGen to work reliably, we need $\rho$ to be close to $1$. For context, we also compute $\rho$ for a standard supervised image classification problem (IC), where train, val and test examples are sampled from the same example-level distribution. 

\textbf{Unreliabity of ValGen snapshot selection.} From Figure \ref{fig:model-selection}(b), we see that when using validation samples generated \iid from the test distribution (Zappos-ID and supervised learning IC),
the value of $\rho$ is consistently higher than the OOD benchmarks, indicating the validation performance can more reliably track the trend of test performance in the ID setting than on the OOD FSL benchmarks.
In particular, for the \cifmod and \FCmod datasets, the ValGen ranking of algorithm snapshots seems to be only weakly correlated with the true meta-test ranking for all the meta-learning methods.
In fact, the meta-val and meta-test rankings of the most useful snapshots can sometimes even be negatively correlated ($\rho$ $\approx$ $-0.12$ $<$ $0$ over all snapshots after epoch 30 in the training scenario shown in Figure~\ref{fig:model-selection}(a)). 
These results show that on the OOD FSL benchmarks, snapshot selection using the pre-assigned validation tasks can sometimes be \textbf{unreliable/unable to identify a snapshot candidate with top-tier meta-test performance among all the snapshots.}

\subsubsection{\textit{Option 2}: Snapshot selection alternative using BaseGen.} 

Beyond using OOD validation tasks for snapshot selection, inspired by the domain generalization community \cite{koh2020wilds, gulrajani2020search}, we can alternatively also consider using the ID performance over the training task distribution for snapshot selection. Perhaps due to the lack of an ID evaluation setup in common FSL benchmarks, this possibility has not been widely considered. 
Enabled by our modifications of current FSL benchmarks, we can now evaluate the ID generalization performance (BaseGen) throughout training in addition to ValGen. 

We plot how BaseGen and NovelGen progress for meta-learning methods trained on two different datasets in \ref{fig:model-selection}(c)(d). Here we see that on \tieredmod, the BaseGen and NovelGen of \pn both increase fairly consistently; thus picking the algorithm snapshot with the highest BaseGen performance (roughly the end-of-training snapshot) would also give close-to-best NovelGen. 
However, on \cifmod, after the learning rate drop at epoch 20, \svm's BaseGen keeps improving while NovelGen starts deteriorating. In this case, selecting snapshots according to the best BaseGen would pick a much worse snapshot than picking according to the best ValGen.
(For concrete numbers of how much snapshot selection through BaseGen vs. ValGen could impact the chosen snapshot's NovelGen performance in each of these two cases, see Appendix~\ref{app:subsec:basegen-vs-valgen}.) This ambiguity of whether ID or OOD Validation snapshot selection is better has also been documented in domain generalization, where
\citet{gulrajani2020search} find ID model selection can perform better in some settings while \citet{koh2020wilds} find OOD validation model selection is better in others. Despite this ambiguity, we believe the commonly neglected \textbf{in-distribution (BaseGen) snapshot selection approach should be considered by users of OOD FSL benchmarks as a viable alternative} to the default ValGen selection approach in proper settings.

\vspace{-0.3em}
\subsection{Inconsistencies in Meta-learning Method Performance Comparisons}
\label{subsec:inconsistencies-in-OOD}
\vspace{-0.3em}

After discussing concerns regarding OOD model selection for \textit{a given meta-learning method}, we now analyze the reliability and consistency of conclusions drawn from comparing \textit{different meta-learning methods}' OOD performance on these benchmarks. In particular, we focus on two cases:

\textbf{\textit{Inconsistency example 1}: Limited number of novel classes in a single benchmark.} Since in OOD FSL we specifically care about the learned algorithms’ ability to quickly learn many unseen concepts, we should not be satisfied with an algorithm performing well only on tasks constructed from a small number of pre-selected novel classes. However, for many widely-used FSL benchmarks (\minipaper, \cifpaper, \FCpaper), only $20$ novel classes are used for meta-testing. Ideally, even if we don’t measure it, we would hope that our learned algorithm would also generalize to tasks made by other sets of classes different from the fixed small set of novel classes.

\textbf{Formal setup.} We suppose the existence of a much larger collection of classes $\calC_L$, where the novel classes $\calC_N$ used for meta-testing is a small random subset with each element class sampled uniformly and non-repeatedly from $\calC_L$ and fixed thereafter for NovelGen evaluation. Ideally, our goal is to evaluate an algorithm snapshot $\phi$ on the task distribution $\Prob_{\calC_L}$ (denote this performance by $A_\phi(\calC_L)$), yet during evaluation we only have access to the novel classes in $\calC_N$ and thus we can only compute the performance $A_\phi(\calC_N)$. It is easy to see that when we randomize over different choices of $\calC_N$, the expected performance over the sampled novel classes would match the true performance: $\E_{\calC_N} [A_\phi(\calC_N)] = A_\phi(\calC_L)$. However, when using a single randomly sampled novel set, the estimator $A_\phi(\calC_N)$ can have high variance (see Appendix~\ref{app:subsec:additional-ood-var}). 
Instead of relying on $A_\phi(\calC_N)$ to directly estimate $A_\phi(\calC_L)$, we ask a more relaxed question: for a pair of algorithms $A_{\phi_1}$ and $A_{\phi_2}$ (given by two meta-learning methods), if the true performance $A_{\phi_1}(\calC_L) - A_{\phi_2}(\calC_L) = \epsilon > 0$, how frequently will we observe an opposite conclusion \ie $\Prob(A_{\phi_1}(\calC_N) < A_{\phi_2}(\calC_N))$ over a randomly sampled $\calC_N$ (we call this event \textit{conclusion flip})?
Additionally, it is also possible that the observed performance difference on $\calC_N$ is greater than the true difference $\epsilon$ by some amount $\delta > 0$. In this case, the NovelGen observation would make the algorithm snapshot $\phi_1$ look better than $\phi_2$ more than it actually is on $\calC_L$. Thus we also ask what is the probability of $\Prob(A_{\phi_1}(\calC_N) - A_{\phi_2}(\calC_N) > \epsilon + \delta)$ and we call such events \textit{improvement exaggerations}. To answer both these questions empirically, we first suggest some larger class sets $\calC_L$ for \mini and \tiered.
For both we select unused classes from ImageNet disjoint from the base and validation classes. For \tiered, we use all the remaining $1000 - 351$ (base) $- 97$ (novel) $= 552$ classes
as $\calC_L$ while for \mini, we randomly choose $120$ or $552$ (to match $|\calC_L|$ in \tiered) unused classes. We fix these $\calC_L$ choices in the following analysis.

\begin{figure}[t]
    \hspace{2.5em}
    \begin{minipage}[b][][b]{0.22\textwidth}
        \setlength{\abovecaptionskip}{2pt} 
        \includegraphics[width=\linewidth]{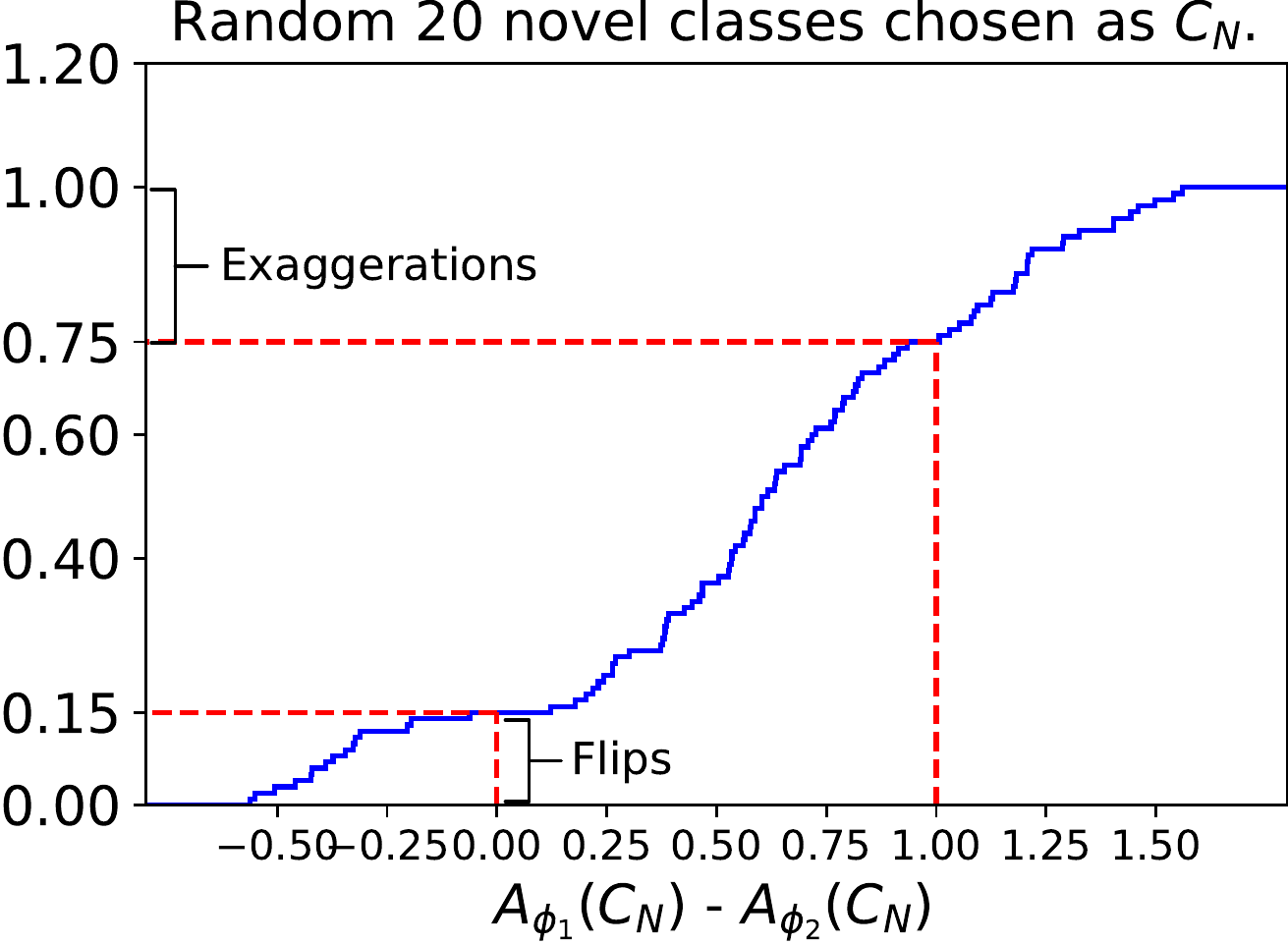}
        \caption*{\;\scriptsize (a)}
    \end{minipage}
    \begin{minipage}[b][][b]{0.36\textwidth}
    \setlength{\abovecaptionskip}{2pt} 
        \scriptsize
        \centering
        \setlength\tabcolsep{1pt}
    \begin{tabular}{c|c|c|ccc|cc}
     \multicolumn{1}{c}{} & \multicolumn{1}{c}{}  & \multicolumn{1}{c}{$\epsilon$} & \multicolumn{1}{c}{$\scriptscriptstyle |\calC_{\scalebox{.7}{$\scriptscriptstyle B$}}|$} & $\scriptscriptstyle |\calC_{\scalebox{.7}{$\scriptscriptstyle N$}}|$ & \multicolumn{1}{c}{$\scriptscriptstyle |\calC_{\scalebox{.7}{$\scriptscriptstyle L$}}|$} & {\tiny CF} & {\tiny IE,$\scriptscriptstyle 0.5\%$} \\ \midrule
    (IC) & \multirow{4}{*}{\rotatebox{90}{\minimod}} & 0.5\% & \multicolumn{3}{c|}{100 imgs/class} & 3\% & 1\% \\
    (i) & & 0.5\% & 64 & 20 & 120 & 15\% & 25\% \\
    (ii) & & 0.5\% & 64 & 20 & 552 & 20\% & 24\% \\
    (iii) & & 1.0\% & 64 & 20 & 552 & 10\% & 12\% \\ \midrule
    (iv) & \multirow{4}{*}{\rotatebox{90}{\tieredmod}} & 0.5\% & 64 & 20 & 552 & 20\% & 23\% \\
    (v) & & 0.5\% & 64 & 160 & 552  & 0\% & 7\% \\
    (vi) & & 0.5\% & 351 & 20 & 552 & 9\% & 15\% \\
    (vii) & & 0.5\% & 351 & 160 & 552 & 0\%  & 1\%  \\\bottomrule
    \end{tabular}
    \caption*{\scriptsize (b)}
    \end{minipage}
    \hspace{-2em}
    \begin{minipage}[b][][b]{0.36\textwidth}
    \setlength{\abovecaptionskip}{2pt} 
        \tiny
        \centering
        \setlength\tabcolsep{2pt}
    \begin{tabular}{c||c|c|c} 
    {\tiny Dataset}  & \multicolumn{2}{c|}{ZAPPOS-OOD} & \mini-OOD \\ \hline
    Methods & 2w10s & 2w5s & 5w5s \\ \Xhline{2\arrayrulewidth}
    {\tiny PN} &
    {\vtop{\hbox{\fontsize{6pt}{6pt} \strut $^{\mathbf{(1)}}80.51$ }\hbox{\strut \fontsize{0.5pt}{0.5pt}$\quad\, \pm0.13 \%$}}} &
    
    {\vtop{\hbox{\fontsize{6pt}{6pt} \strut $^{\mathbf{(1)}}74.67$ }\hbox{\strut \fontsize{0.5pt}{0.5pt}$\quad\, \pm0.14 \%$}}} &
    
    {\vtop{\hbox{\fontsize{6pt}{6pt} \strut $^{\mathbf{(3)}}76.22$ }\hbox{\strut \fontsize{0.5pt}{0.5pt}$\quad\, \pm0.14 \%$}}} \\
    
    {\tiny  Ridge} &
    
    {\vtop{\hbox{\fontsize{6pt}{6pt} \strut $^{\mathbf{(1)}}80.51$ }\hbox{\strut \fontsize{0.5pt}{0.5pt}$\quad\, \pm0.13 \%$}}} &
    
    {\vtop{\hbox{\fontsize{6pt}{6pt} \strut $^{\mathbf{(2)}}73.50$ }\hbox{\strut \fontsize{0.5pt}{0.5pt}$\quad\, \pm0.14 \%$}}} &
    
    {\vtop{\hbox{\fontsize{6pt}{6pt} \strut $^{\mathbf{(2)}}77.20$ }\hbox{\strut \fontsize{0.5pt}{0.5pt}$\quad\, \pm0.15 \%$}}} \\
    
    {\tiny SVM} &
    
    {\vtop{\hbox{\fontsize{6pt}{6pt} \strut $^{\mathbf{(3)}}79.70$ }\hbox{\strut \fontsize{0.5pt}{0.5pt}$\quad\, \pm0.13 \%$}}} &
    
    {\vtop{\hbox{\fontsize{6pt}{6pt} \strut $^{\mathbf{(3)}}72.84$ }\hbox{\strut \fontsize{0.5pt}{0.5pt}$\quad\, \pm0.14 \%$}}} &
    
    {\vtop{\hbox{\fontsize{6pt}{6pt} \strut $^{\mathbf{(1)}}77.72$ }\hbox{\strut \fontsize{0.5pt}{0.5pt}$\quad\, \pm0.15 \%$}}} \\

     \vtop{\hbox{\fontsize{6pt}{6pt} \;\;\;FO-}\hbox{\fontsize{6pt}{6pt} MAML}}&
    {\vtop{\fontsize{6pt}{6pt} \hbox{\strut $^{\mathbf{(4)}}72.89$ }\hbox{\strut \fontsize{0.5pt}{0.5pt}$\quad\, \pm0.14 \%$}}} &
    
    {\vtop{\hbox{\strut $^{\mathbf{(4)}}68.12$ }\hbox{\strut \fontsize{0.5pt}{0.5pt}$\quad\, \pm0.15 \%$}}} &
    
    {\vtop{\hbox{\strut $^{\mathbf{(4)}}75.96$ }\hbox{\strut \fontsize{0.5pt}{0.5pt}$\quad\, \pm0.17 \%$}}} \\\bottomrule
    \end{tabular}
    \caption*{\;\;\scriptsize(c)}
    \end{minipage} \hfill
    \vspace{1em}
    \caption{In (a), we show the CDF plot of $A_{\phi_1}(\calC_N)-A_{\phi_2}(\calC_N)$ over 100 randomly chosen sets of $\calC_N$ with 20 (out of total $|\calC_L|$$=$$120$) novel classes each. In (b), for different values of true performance difference $\epsilon$ and values of underlying class size $|\calC_L|$, training class size ($|\calC_B|$), and evaluation class size ($|\calC_N|$); we show the percentage of conclusion flips (CF) and improvement exaggerations (IE) with $\delta=0.5\%$ computed over 100 evaluations.
    In (c), we demonstrate the inconsistencies in performance rankings for \pn, \svm, \ridge, and \fomaml on two OOD benchmarks: Zappos-OOD and \mini-OOD. 
    } 
    \label{fig:conclusion_inconsistency}
    \vspace{-0.3em}
\end{figure}
\normalsize

\textbf{Checking the frequency of conclusion flips and exaggerations.} 
Figure~\ref{fig:conclusion_inconsistency}(a) shows the empirical CDF of the performance differences $A_{\phi_1}(\calC_N) - A_{\phi_2}(\calC_N)$ computed over $100$ randomly sampled size-20 novel class sets $\calC_N$ for a fixed pair of \pn and \ridge algorithm snapshots whose true performance difference over the larger 120 classes $\calC_L$ is $\epsilon=0.5\%$. In 15\% of the cases the performance order is flipped from the true order, while in 25\% of them improvements are exaggerated by more than $0.5\%$ (total difference greater than $1\%$). Moreover, for some of the performance order flips, the observed performance difference can be quite negative $<$ $-0.5\%$ thus significantly opposite to the true performance order.
(Here for each run we evaluate both methods on 20,000 tasks sampled from $\Prob_{\calC_N}$ in order to significantly reduce the randomness in estimating the true $A_{\phi_1}(\calC_N), A_{\phi_2}(\calC_N)$.)

\textbf{Comparison to supervised learning.} We check for the conclusion flip and improvement exaggeration frequency when only a random subset of the full test set (100 randomly drawn test images from each base class in \minimod) is used to compare two supervised learning image classification models 
with the same full test set performance difference of $\epsilon=0.5$ (row (IC) in Table \ref{fig:conclusion_inconsistency}(b)).
Here we see that \textbf{compared to supervised learning, the chances of getting an incorrect performance comparison {\normalfont (row (i) in Table~\ref{fig:conclusion_inconsistency}(b))} is much higher for the meta-learning OOD FSL benchmarks} when evaluating only on 20 randomly chosen novel classes (as done in several FSL benchmarks).

\textbf{Larger $\abs{\calC_L}$ makes it even less reliable but larger $\epsilon$ helps.} If we were to care about an even larger set of underlying classes ($|\calC_L|=552$) despite still using only 20 random novel classes for evaluation comparison, the conclusions are even less reliable (Table~\ref{fig:conclusion_inconsistency}(b) (i) vs (ii)).
On the other hand, we do see that the performance comparison becomes comparatively more consistent if the true performance difference $\epsilon$ is higher ($1\%$ in (iii) compared to $0.5\%$ in (ii)), despite that there still exists a statistically significant chance ($10\%$) of getting an opposite conclusion. 

\textbf{OOD evaluations in current FSL benchmarks.} In practice, because \textbf{1)} we never specify exactly what and how big the underlying set of classes that we care about is, and \textbf{2)} some of the recent meta-learning methods (\svm vs \pn on \cif in Table 2 of \cite{lee2019meta}, \textsc{r2-d2} vs \textsc{gnn} on \mini in Table 1 of \cite{bertinetto2018meta}, \fixml \cite{setlur2020support}) sometimes only improve over the prior works by $< 1\%$, we believe researchers should \textbf{be aware of the possibility of getting a performance conclusion that is inconsistent} over a single randomly chosen and fixed set of 20 novel classes used by some of these benchmarks.

\textbf{Actionable suggestions.} Since the size of the unknown underlying larger class set $\calC_L$ and the true performance difference $\epsilon$ might not be something one can directly control when designing the OOD benchmark, we now discuss two actionable choices that can reduce the chances of conclusion flips: 
\begin{enumerate}[label=\roman*),topsep=0pt,itemsep=-1ex,partopsep=1ex,parsep=1ex, leftmargin=15pt]
    \item \textbf{Use more novel classes in the evaluation:} By comparing (iv) vs (v) and (vi) vs (vii) in Table~\ref{fig:conclusion_inconsistency}, we see that the frequency of conclusion flips and improvement exaggerations are \textbf{much} lower when 160 novel classes are used as opposed to 20 when $|\calC_L|$ is the same.
    \item \textbf{Train on more base classes:} The \tiered dataset has more base classes ($351$ compared to $64$ for \mini) to train on. When comparing \pn and \ridge snapshots trained on a modified version of \tiered with fewer (randomly sampled 64 out of 351 to match \mini) base classes, we see that the CF frequency is twice as high compared to when $351$ base classes are used (Table~\ref{fig:conclusion_inconsistency}(b)(iv) vs (vi)).
\end{enumerate}
Based on these two trends, for more reliable comparisons of meta-learning methods' OOD performance we suggest using datasets like \tieredpaper and MetaDataset (both with much larger set of base and novel classes) 
in addition to the smaller benchmarks like \minipaper, \cifpaper, and \FCpaper, which some recent works \cite[e.g.,][]{oreshkin2018tadam,bertinetto2018meta} still solely rely upon.

\textbf{\textit{Inconsistency example 2}: Inconsistency across multiple OOD FSL benchmarks.}
Unlike the ID scenario where the training and test task distribution are the same, 
the similarity between training and test distributions in the OOD FSL benchmarks can vary significantly. Ideally, we want a meta-learning method to be consistently better on multiple OOD benchmarks with different type/degree of distribution shifts.
Since \citet{ren2020flexible} originally use the Zappos dataset for OOD evaluation, we also perform a similar evaluation on new attribute pairs based on their setup. At test time, we use an attribute set $\calA'$ disjoint from the one used in the Zappos-ID setup $\calA$, and sample attribute pairs from $\calA'$ only. This induces a test task distribution $\Prob_{\calA'}$ different from the training task distribution $\Prob_{\calA}$.
We evaluate different meta-learning methods on these Zappos-OOD tasks to see if the performance order is consistent with other OOD FSL benchmarks (Table~\ref{fig:conclusion_inconsistency}(c)).
Here we see that despite \svm outperforming \ridge and \pn on \mini NovelGen, the performance order of these three methods are completely flipped on Zappos-OOD.
Similar observations can be made from TADAM underperforming \pn in Table~2 of \cite{ren2020flexible} despite TADAM being shown to outperform \pn on the other more commonly-used FSL benchmarks.
\textit{This inconsistency over different types of OOD FSL benchmarks is in stark contrast to the consistency of performance rankings over the $6$ different ID benchmarks} (FEMNIST, Zappos-ID, and the BaseGen results of the $4$ current FSL benchmarks (Section~\ref{sec:ID})). Based on these findings, we caution meta-learning researchers to \textbf{be aware of such conclusion inconsistencies over different OOD FSL scenarios} and \textbf{reason carefully about the generality of their empirical findings} when using only specific types of OOD datasets.

\vspace{-8pt}
\section{Conclusion}
\label{sec:discussion}
\vspace{-10pt}
In this paper, we categorize meta few-shot learning evaluation into two settings: in-distribution (ID) and out-of-distribution (OOD). After explaining why common FSL benchmarks reflect OOD evaluation, we identify realistic needs for ID FSL evaluation and provide new benchmarks as well as suggestions on how to modify existing OOD FSL benchmarks to allow for ID evaluation. Through experiments performed on these ID benchmarks, we demonstrate a surprising phenomenon that many meta-learning methods/training techniques improve OOD performance while sacrificing ID performance.  
Beyond this, 
through quantitative analyses, we show that even in the OOD scenario, current FSL benchmarks may present subtle challenges with both model selection for a given meta-learning method and reliable performance comparisons of different methods. For these concerns, we provide initial suggestions and alternatives with the hope of alleviating these issues. Overall, we aim to raise awareness about the dichotomy of FSL evaluation and to motivate the meta-learning community to collectively reason about ways to improve both ID and OOD methodology and evaluation.

\newpage
\paragraph{Acknowledgements.}

This work was supported in part by the
National Science Foundation Grant IIS1838017, a Google
Faculty Award, a Facebook Faculty Award, and the CONIX
Research Center. Any opinions, findings, and conclusions
or recommendations expressed in this material are those
of the author(s) and do not necessarily reflect the NSF or any other funding agency.

\bibliographystyle{abbrvnat}
\bibliography{main}

\appendix
\onecolumn
\section*{Appendix}
\subsection*{Appendix Outline}
\ref{sec:quote-meta-learning-methods}. Assumption (ID) and Evaluation (OOD) Mismatch Examples

\ref{app:formal-analysis}. Formal Analysis on Why Current FSL Benchmarks Are OOD

\ref{app:sample-notation}. Overview of Notations for $(S, Q)$ Sampling in ID and OOD Benchmarks

\ref{app:zappos-and-hyperparameter}. Dataset Preprocessing and Hyperparameter Details

\ref{app:additional-id}. Additional Results on Evaluating ID Performance

\ref{app:additional-ood}. Additional Discussion and Results on OOD Evaluation

\section{Assumption (ID) and Evaluation (OOD) Mismatch Examples} 
\label{sec:quote-meta-learning-methods}

To illustrate the mismatch between meta-learning theory/methodology and evaluation (Section~\ref{sec:relwork}), below are some examples of works that motivate commonly-used meta-learning methods in the in-distribution setting, but largely evaluate empirical performance on OOD FSL benchmarks. For convenience, we provide exact lines from the original works that refer to the ID scenario. Our aim is not to draw attention to these works specifically, but to highlight the ubiquity of the divide between theory and practice in current meta-learning literature.

\vspace{-0.1in}
\begin{itemize}[leftmargin=20pt]
\setlength\itemsep{0.1em}
\item  \citet{lee2019meta} (Section 3.1): ``
It is often \textbf{assumed that the
training and test set are sampled from the same distribution}
and the domain is mapped to a feature space using an embedding model $f_\phi$ parameterized by $\phi$
''.


\item  \citet{rajeswaran2019meta} (Section 2.1): ``... a collection of \textbf{meta-training} tasks $\{\calT_i\}_{i=1}^M$ \textbf{drawn from} $P(\calT)$ ... At \textbf{meta-test} (deployment) time, when presented with a dataset $\calD^\tr_j$ corresponding to a \textbf{new task} $\calT_j \sim P(\calT)$.''
Notice that the training and test tasks are all sampled from the same task distribution $P(\calT)$.

\item \citet{finn2018probabilistic} (Section 3): ``
To do so, meta-learning algorithms require a set of \textbf{meta-training and meta-testing tasks drawn from some distribution $p(\calT)$}. The key assumption of learning-to-learn is that the tasks in this distribution share common ...
''.

\end{itemize}

\section{Formal Analysis on Why Current FSL Benchmarks Are OOD}

\label{app:formal-analysis}

In this section we provide formal arguments for the informal statements in Section~\ref{sec:background}, which explain why it is  improbable for the train and test tasks in the current FSL benchmarks to be \iid sampled from the same underlying distribution. 

\textbf{Formal Setup.} If we believe that both train and test tasks in current FSL benchmarks are sampled from the same underlying task distribution, then this shared task distribution (where each task is specified by a class tuple) must cover a larger set of underlying classes $\calC_L$ which would contain both the base classes and novel classes as subsets $\calC_L \supseteq (\calC_B \cup \calC_N)$. For convenience, we represent the classes in this set with  $\calC_L \coloneqq \{1, \ldots, L\}$ where the task distribution (from which the train and test tasks are \iid sampled) is induced by a probabilistic distribution over $n-$way non-repeating tuples $\bc \coloneqq (c_1, \ldots, c_n) \in [\calC_L^n]$, denoted by  $\Prob_L(\bc)$. To sample a task from this larger task distribution, we sample $\bc \sim \Prob_L$ and take the corresponding task $\calT_{\bc}$. Notice that this task distribution can be more general than $\Prob_{\calC_N}$ or $\Prob_{\calC_B}$, as $\Prob_L(\bc)$ does not have to be a uniform distribution over all possible class tuples.

\begin{definition}[Probability of observing a class in a single draw]
The indicator event of observing a class $i$ anywhere in a randomly drawn class tuple $\bc \sim \Prob_L$ can be represented by 
$\sum_{j=1}^n \I(c_j=i)$, since it is impossible to observe the same class more than once in the same tuple. We denote the probability of this event by $p_i \coloneqq \Prob\paren{\paren{\sum_{j=1}^n \I(c_j=i)} = 1} = \sum_{j=1}^n \Prob(\I(c_j=i) = 1) = \sum_{j=1}^n \Prob(c_j=i)$.
\label{defn:pi}
\end{definition}

\begin{lemma} The sum of the probability of observing a class $i$ in a single class tuple draw over all the classes $i \in \calC_L$ is equal to $n$, \ie $\sum_{i=1}^L p_i = n$.
\label{lemma:pi-sum}
\end{lemma}
\begin{proof}[Proof of Lemma~\ref{lemma:pi-sum}]
 We know that $\sum_{i=1}^L p_i = \sum_{i=1}^L \sum_{j=1}^n  \Prob(c_j=i)$ by simply plugging in the definition of $p_i$. Since $\sum_{i=1}^L \Prob(c_j=i)=1$, by exchanging the summations we get,  $\sum_{i=1}^L p_i = \sum_{j=1}^n \sum_{i=1}^L \Prob(c_j=i) =  \sum_{j=1}^n [1] = n$. 
\end{proof}

\begin{assumption}[Every class must have nonzero probability to be sampled]
To avoid degeneracy, we assume that each class has a minimum non-zero probability of being sampled in a class tuple: $\forall i \in \{1, \ldots, L\},\, 1\geq p_i \geq \frac{\gamma n}{L}$, where $\gamma \in (0,1]$. Notice $\gamma$ is strictly greater than 0 to avoid the degenerate case where a class would almost surely never be sampled in any class tuple. If there exists such a class, then we can prune the set $\calC_L$ accordingly and use the pruned set (which now has every class with nonzero probability) as our new $\calC_L$.
\label{assm:pi}
\end{assumption}

\textbf{Remark.} Note that the task distribution induced by the probability values $\{p_i\}_{i=1}^L$ in  Assumption~\ref{assm:pi} is a relaxed form of the uniform distribution $\Unif(
\{\calT_{\bc} \coloneqq \bc \in [\calC_L^n]\})$ over all non-repeating class tuples spawned by $\calC_L$. 
This case can be recovered by setting
$p_i = \frac{n}{L},\; \forall i \in [L]$.







Suppose there are $N$ total \iid random draws $\{\bc^{(k)}\}_{k=1}^N$ of class tuples from $\Prob_L$ (every $\bc^{(k)} \in [\calC_L^n]$), then the event of observing a class $i$ in any of these $N$ class tuple draws is exactly the complement of the event that the class does not appear in any of these tuples. 

\begin{definition}[Observing a class at least once in $N$ draws]
We denote the indicator random variable of observing a class $i$ in any of the $N$ draws by 
\begin{align}
 X_{i,N} \coloneqq 1 - \I\paren{\paren{\sum_{k=1}^N \I(i \in \bc^{(k)})} = 0} \in \{0, 1\}   .
 \label{eq:xin-defn}
\end{align}
Then we have $\E[X_{i,N}] = \Prob(X_{i,N} = 1) = 1- (1-p_i)^N$. We denote the random variable representing the total number of unique classes observed in $N$ draws as $Z$, which can be expressed by
\begin{align}
 Z = \sum_{i=1}^L X_{i,N}    .
\label{eq:z-defn}
\end{align}
\label{defn:xin-z}
\end{definition}
\vspace{-.1in}

\textbf{Remark.} We note that the total number of unique classes seen $(Z)$ in $N$ \iid draws \textbf{1)} must have at least $n$ classes (even after a single class tuple is sampled, there would already be $n$ different classes seen) and cannot be greater than the total number of classes possible, \ie $Z \in [n, L]$, and \textbf{2)} cannot be greater than the total number of (possibly overlapping) classes drawn, \ie $Z \le nN$.

\begin{lemma} For notational convenience, let $q_i \coloneqq 1 - p_i$. Then, by Equation~\eqref{eq:xin-defn} and Assumption~\ref{assm:pi} we have:
\begin{enumerate}[label=(\alph*)]
    \item  $\E[Z] = L - \sum_{i=1}^L (1-p_i)^N = L - \sum_{i=1}^L q_i^N, \qquad \qquad \quad \; ~\refstepcounter{equation}(\theequation)\label{eq:ez}$
    \item  For $\{q_i\}_{i=1}^L$, $0 \leq\; q_i \leq\; 1 - \frac{\gamma n}{L}$ and  
    $\sum_{i=1}^L q_i = L - n.  \qquad \; ~\refstepcounter{equation}(\theequation)\label{eq:q-condition}$
\end{enumerate}
\label{lemma:EZ}
\end{lemma}

Now that we have set up the problem formulation, we provide Theorems~\ref{thm:lower-bound-ez},~\ref{thm:upper-bound-vz} describing properties of $\E[Z]$ and $\V[Z]$ which we will use to analyze the dichotomy described in the main paper (Section~\ref{sec:background}).

\subsection{Lower Bound on \texorpdfstring{$\E[Z]$}{EZ}}

To achieve a lower bound of $\E[Z]$, we need to analyze the worst case class tuple distribution that makes the value $L - \sum_{i=1}^L q_i^N$ as small as possible. This amounts to maximizing the value of $\sum_{i=1}^L q_i^N$ under the constraints for $\{q_i\}_{i=1}^L$ described in Lemma~\ref{lemma:EZ}. We present an upper bound for this constrained maximization objective below. 

\begin{theorem}[Lower bound on $\E\brck{Z}$]
The optimal value of the following constrained optimization problem in \eqref{eq:q-objective} is upper bounded by $L\paren{1-\frac{\gamma n}{L}}^N$. 
\begin{align}
    & \max_{\{q_i\}_{i=1}^L} \quad \sum_{i=1}^L q_i^N \label{eq:q-objective} \\
    \textrm{subject to}\quad  &    0 \le\; q_i \le \; 1 - \frac{\gamma n}{L}, \; \forall i \in [L] \nonumber \\
    & \quad \sum_{i=1}^L q_i = L - n \nonumber 
\end{align}
As a result, for $Z$ defined in \eqref{eq:z-defn}, directly applying \eqref{eq:ez} we get $\E[Z] \geq L\paren{1-(1-\frac{\gamma n}{L})^N}$. 
\label{thm:lower-bound-ez}
\end{theorem}

To prove Theorem~\ref{thm:lower-bound-ez}, we first provide a lemma describing the structure of the solution to the optimization problem in \eqref{eq:q-objective}.
\begin{lemma}[Structure of the optimal solution to \eqref{eq:q-objective}]
 \label{lemma_numK}
 The optimal solution to optimization Objective \eqref{eq:q-objective} has the following form: out of the $L$ variables $\{q_i\}_{i=1}^L$, $K$ of them have value $1-\frac{\gamma n}{L}$, $(L-K-1)$ of them have value $0$, and the last remaining variable has the value $(L - n) - K(1-\frac{\gamma n}{L})$, which must still be in the range of $\brck{0, 1 - \frac{\gamma n}{L}}$. This directly implies that the integer $K$ must be satisfy $\frac{L^2 - nL}{L - \gamma n} - 1 \leq K \leq \frac{L^2 - nL}{L - \gamma n}$.
\end{lemma}

\begin{proof}[Proof of Lemma~\ref{lemma_numK}]
Let us denote the optimal solution to Objective~\eqref{eq:q-objective} by $(q_1^*, \ldots, q_n^*)$. Suppose that there exists a pair $q_k^*$, $q_j^*,\; k\neq j$, such that neither of them equals $0$ or $1-\frac{\gamma n}{L}$. Then by changing the values of $q_k^*$, $q_j^*$ to be either $(1-\frac{\gamma n}{L}, q_k^* + q_j^* - 1 + \frac{\gamma n}{L})$ or $(0, q_k^* + q_j^*)$, the new $q$ tuple would still be feasible while the value of Objective~\ref{eq:q-objective} would strictly improve because the function $q_k^N + q_j^N$ is strongly convex over $\Real_{++}^2$. (Recall that a convex function over a closed interval can only take maximum value at either one of its two endpoints.) As a result, there can be no more than a single $q_i^*$ in the optimal solution $(q_1^*, \ldots, q_L^*)$ that has a value of neither $0$ nor $1 - \frac{\gamma n}{L}$. Now, we denote the total number of $q_i$'s in the optimal solution that has the value of $1-\frac{\gamma n}{L}$ by $K \in \Integer$, then there must be at least $L-K-1$ values of $0$, with the remaining term $(L - n) - K(1-\frac{\gamma n}{L}) \in \brck{0, 1-\frac{\gamma n}{L}}$. Manipulating this inequality of $K$ gives us the feasible range of $K$, $\frac{L^2 - nL}{L - \gamma n} - 1 \leq K \leq \frac{L^2 - nL}{L - \gamma n}$.
\end{proof}

\begin{proof}[Proof of Theorem~\ref{thm:lower-bound-ez}]
We know from Lemma 3 that the optimal solution to the constrained maximization problem in Theorem~\ref{thm:lower-bound-ez} is given by the optimal solution to the reduced objective below.
\begin{align}
  & \max_{K\in \{0, \ldots, L\}} \quad  K \cdot \paren{1 - \frac{\gamma n}{L}}^N + \brck{L - n - K \cdot \paren{1-\frac{\gamma n}{L}}}^N \label{eq:q-objective-reduced}  \\
   & \textrm{subject to}  \quad \frac{L^2 - nL}{L - \gamma n} - 1 \leq\; K \leq\; \frac{L^2 - nL}{L - \gamma n}
\end{align}
Let the optimal value of $K$ (minimum value if multiple are optimal) in the above optimization problem be $K^*$. Then the optimal value of the above is upper bounded by:
\begin{align*}
&K^* \cdot \paren{1 - \frac{\gamma n}{L}}^N  + \brck{L - n - K^* \cdot \paren{1-\frac{\gamma n}{L}}}^N \\
\le&\; (K^* +1) \cdot \paren{1 - \frac{\gamma n}{L}}^N  \\
\le&\; L \cdot \paren{1 - \frac{\gamma n}{L}}^N 
\end{align*}

As a result, directly applying \eqref{eq:ez}, we have:
\begin{align}
    \E[Z] \ge L \cdot \brck{1 - \paren{1 - \frac{\gamma n}{L}}^N}.
    \label{eq:ez-lb}
\end{align}
\end{proof}

\subsection{Upper bound on \texorpdfstring{$\V[Z]$}{VZ}}

We apply the Efron-Stein inequality to obtain the upper bound of the variance of $Z$, which we state here for convenience.
\begin{lemma}[Efron-Stein's inequality]
Let $S: \calY^N \rightarrow \Real$ be a measurable function that is permutation invariant. Let the random variable $U$ be given by $U = S(Y_1, \ldots, Y_N)$, where $(Y_1, \ldots, Y_N)$ is a random vector of $N$ independent random variables in  $\calY^N$. Then, we have:
\begin{align}
\V[U] \le \frac{1}{2}\sum_{i=1}^N \E (U-U_i')^2,
\end{align}
where $U_i' = S(Y_1, ..., Y_i', ..., Y_N)$, and $\forall i \in [N]$,  $Y_i$ and $Y_i'$ are drawn \iid from the same distribution. 
\label{lemma:efron-stein}
\end{lemma}

\begin{theorem}[Upper bound on $\V\brck{Z}$] 
\label{thm:upper-bound-vz}
For $Z$ defined in \eqref{eq:z-defn}, the variance  $\V[Z] \leq \frac{1}{2} n^2 N$.
\end{theorem}

\begin{proof}[Proof of Theorem~\ref{thm:upper-bound-vz}]
By directly applying Lemma~\ref{lemma:efron-stein} on the permutation invariant measurable function $S: [\calC_L^n]^N \rightarrow \Real$, where $Z=S(\bc^{(1)}, \ldots, \bc^{(N)}) \coloneqq \sum_{i=1}^{L} X_{i,N}$,   
we have the variance $\V[Z] \leq \frac{1}{2}\sum_{i=1}^N \E (Z-Z_i')^2 \le \frac{1}{2}\sum_{i=1}^N n^2 = \frac{1}{2} n^2N$. Note that the last inequality holds because 
when we swap out one observed class tuple $\bc^{(k)}$ (among the $N$ total) with a different one $\bc^{(k)'}$ (to get $Z_i'$), the total number of unique classes we observe can change by at most  $n$, \ie $|Z - Z_i'| \leq n$.
\end{proof}






Using Theorems~\ref{thm:lower-bound-ez},~\ref{thm:upper-bound-vz}, we now show that it is extremely unlikely for the commonly used FSL benchmarks, which have a (relatively) small number of train classes and disjoint train/test classes, to have their training and test tasks sampled \iid from the same underlying task distribution. We break this analysis into the dichotomy presented in the main paper (Section~\ref{sec:background}). 

\subsection{When \texorpdfstring{$|\calC_L|$}{CL} is Small (\texorpdfstring{$L=\calO(nN)$}{})}

\begin{definition}
Let $Z_\tr$ be the total number of unique classes observed in $N_\tr$ \iid drawn train class tuples (tasks) from $\Prob_L$ and similarly let $Z_\te$ be the total number of unique classes observed in $N_\te$ \iid drawn test class tuples (tasks) from $\Prob_L$.  Furthermore, denote the set of indices of unique train classes by $\calI_\tr \coloneqq \{j \st X_{j,N_\tr}=1, j \in [N_\tr]\}$ and the set of unique test classes be $\calI_\te \coloneqq \{j \st X_{j,N_\te}=1, j \in [N_\te]\}$. Under this notation, the probability of observing disjoint sets of train and test classes among the $N_\tr$ and $N_\te$ randomly drawn train and test class tuples can be denoted by $\Prob(\calI_\tr \cap \calI_\te = \phi)$. 
\label{defn:Itr-Ite}
\end{definition}

\begin{theorem}[Upper bounding the probability of having disjoint train, test classes]
\label{thm:case-1} $\Prob(Z_\tr + Z_\te \le L) \le 4\paren{1-\frac{\gamma n}{L}}^{\textrm{min}(N_\tr, N_\te)}$. As a result, $\Prob(\calI_\tr \cap \calI_\te = \phi) \le \; \Prob(Z_\tr + Z_\te \le L) \le \; 4\paren{1-\frac{\gamma n}{L}}^{\textrm{min}(N_\tr, N_\te)}$.

\label{thm:disj-prob-ub}
\end{theorem}

\begin{proof}[Proof of Theorem~\ref{thm:disj-prob-ub}]

Since the random variable $L-Z_\tr \ge 0$, using Markov's inequality we have:

\begin{align}
    \Prob\paren{Z_\tr \le \frac{L}{2}} 
    =& \Prob\paren{L - Z_\tr \ge L - \frac{L}{2}} \\
    \le& \; \frac{L - \E[Z_\tr]}{(L - \frac{L}{2})} \nonumber \\
    \le& \; \frac{2(L - \E[Z_\tr])}{L} \nonumber \\
    \le& \;  2\paren{1-\frac{\gamma n}{L}}^{N_\tr}. \fourquad \paren{\textrm{using Theorem~\ref{thm:lower-bound-ez}}} 
\end{align}

Similarly, since $L-Z_\te \ge 0$, we have

\begin{align}
    \Prob\paren{Z_\te \le \frac{L}{2}} \le& \; 2\paren{1-\frac{\gamma n}{L}}^{N_\te} .
\end{align}

Since $\Prob\paren{Z_\tr + Z_\te \le L} \leq \Prob\paren{\paren{Z_\tr \leq \frac{L}{2}} \cup \paren{Z_\te \leq \frac{L}{2}}}$, applying the union bound yields: 
\begin{align}
    \Prob\paren{Z_\te + Z_\te \le L} \leq 4\paren{1-\frac{\gamma n}{L}}^{\textrm{min}(N_\tr, N_\te)}.
\end{align}

When $Z_\tr + Z_\te > L$, by pigeonhole principle, the two sets of class indices $\calI_\tr$ and $\calI_\te$ must have non-empty intersection, \ie $\calI_\tr \cap \calI_\te \neq \phi$. Taking the contra-positive of this claim, we see that $\calI_\tr \cap \calI_\te = \phi$ implies the event $Z_\tr + Z_\te \le L$. As a result, we have $\Prob(\calI_\tr \cap \calI_\te = \phi) \le \; \Prob(Z_\tr + Z_\te \le L) \le \; 4\paren{1-\frac{\gamma n}{L}}^{\textrm{min}(N_\tr, N_\te)}$.
\end{proof}


\begin{corollary}
If enough samples are observed \ie if $\min(N_{\tr}, N_{\te}) \ge \frac{\ln(4/\rho)L}{\gamma n}$, then the probability of having no training and test classes intersection is upper bounded by $\Prob(\calI_\tr \cap \calI_\te = \phi) \le \rho$.
\label{corr:no_intersection}
\end{corollary}
\begin{proof}[Proof of Corollary \ref{corr:no_intersection}]
By logarithm inequality $\ln(1 + x) \ge \frac{x}{1+x}$, we have $\ln(L/(L - \gamma n)) \ge \gamma n / L$. Taking the reciprocal of the two sides, we have $\frac{L}{\gamma n} \ge \frac{1}{\ln(L / (L - \gamma n)}$. As a result, $\min(N_{\tr}, N_{\te}) \ge \frac{\ln(4/\rho)L}{\gamma n} \ge \frac{\ln(4/\rho)}{\ln(L/(L - \gamma n))} = \frac{\ln(\rho/4)}{\ln(1 - \frac{\gamma n}{L})} = \log_{(1 - \frac{\gamma n}{L})} (\frac{\rho}{4})$, where the last step uses the change of basis equality of logarithm. Thus, by $\min(N_{\tr}, N_{\te}) \ge \log_{(1 - \frac{\gamma n}{L})} (\frac{\rho}{4})$, we have
\begin{align}
    &\Prob(\calI_\tr \cap \calI_\te = \phi) \nonumber \\
    \le& \; \Prob(Z_\tr + Z_\te \le L) \nonumber \\
    \le& \; 4\paren{1-\frac{\gamma n}{L}}^{\textrm{min}(N_\tr, N_\te)} \nonumber \\
    \le& \; 4\paren{1-\frac{\gamma n}{L}}^{\log_{(1 - \frac{\gamma n}{L})} (\frac{\rho}{4})} \nonumber \\
    =& \; 4 \cdot \frac{\rho}{4} \nonumber \\
    =& \; \rho \nonumber.
\end{align}
\end{proof}

\textbf{Remark.} From Corollary~\ref{corr:no_intersection}, we see that when the number of tasks sampled is larger than a multiple of the number of underlying classes (for example when $\rho=0.01$, $\gamma=0.5$, $n=5$, $\min(N_\tr, N_\te) \ge \frac{\ln(4/\rho)L}{\gamma n} \approx 2.39L$), and equivalently, $L  = \calO(nN)$, the probability of having no training and test task classes intersecting is upper bounded by $\Prob(\calI_\tr \cap \calI_\te = \phi) \le \rho$ (in our example, the probability is upper bounded by $\rho = 0.01$, which is a statistically rare event). In summary, in this case, we show that when $L \le c nN$ for some small constant $c$, the probability of having no intersection between the training and test task classes is extremely small because it is very likely that the training tasks and test tasks would each cover a majority ($\ge 50\%$) of the entire set of classes.

\subsection{When \texorpdfstring{$|\calC_L|$}{CL} is Large (\texorpdfstring{$L=\Omega(nN)$}{})}

In this alternate case, we analyze the scenario where the underlying set of classes is larger than the total number of tasks we sampled, for which we make the following assumption:
\begin{assumption} $L \ge nN$, \ie even if we observe all the classes in the randomly drawn $N$ tuples to be distinct, we still would not exhaust the much larger underlying
set $\calC_L$. In this setting, $L=\Omega(nN)$.
\label{assm:nN}
\end{assumption}

\begin{corollary}
By Assumption~\ref{assm:nN} and the Bernoulli inequality $(1 + x)^r \le 1 + \frac{rx}{1-(r-1)x}, \; x\in(-1, \frac{1}{r-1}), r>1$, substituting $x=\frac{-\gamma n }{L}, \; r=N$, we can further lower bound the RHS of Equation~\eqref{eq:ez-lb} in Theorem~\ref{thm:lower-bound-ez}:
\begin{align}
    \E[Z] \geq \frac{\gamma n N}{(1+\gamma)} 
\end{align}
\label{corr:ez-final-lb}
\end{corollary}

\begin{theorem}[Unlikely to observe only a small number of unique classes]
For $\eta \in (0,1)$, the probability of observing totally fewer than $\frac{\eta \gamma n N}{1+\gamma}$ classes in $N$ \iid class tuple samples from $\Prob_L$ is at most $\frac{(1 + \gamma)^2 }{2(1 - \eta)^2 \gamma^2N}$.  
\label{thm:case-2}
\end{theorem}

\begin{proof}
\begin{align}
    & \Prob\paren{Z \leq  \frac{\eta \gamma nN}{(1 + \gamma)}} \\
    \le& \Prob(Z \leq \eta\E[Z]) \nonumber\\
    =&\;  \Prob(Z \leq \E[Z] - (1 - \eta)\E[Z]) \nonumber\\
    \le&\; \Prob(|Z - \E[Z]| \ge (1 - \eta)\E[Z]) \nonumber \\
    \le&\; \frac{\V[Z]}{((1 - \eta)\E [Z])^2} \qquad\quad\; \textrm{(by Chebyshev's inequality)} \nonumber \\
    \le&\;  \frac{\frac{1}{2}n^2N}{(1-\eta)^2 \E[Z]^2} \nonumber \qquad\quad\;\; \textrm{(using Theorem~\ref{thm:upper-bound-vz})} \\
    \le&\;  \frac{\frac{1}{2}n^2N}{(1-\eta)^2 \paren{\frac{\gamma nN}{(1 + \gamma)}}^2} \nonumber \qquad \textrm{(using Corollary~\ref{corr:ez-final-lb})}\\
    =&\; \frac{(1 + \gamma)^2 }{2(1 - \eta)^2 \gamma^2N}
\end{align}
\end{proof}

In summary, in this case, when $L \ge nN$, the probability of observing only a small fraction of $nN$ classes in $N$ tuple draws, scales with $1/N$. Because in practice a very large number of training tasks are used ($N \ge 10^5$), the probability of only observing fewer than hundreds of classes ($Z \le 10^3$) in $N$ class tuple samples would be extremely small. This means that we shouldn't treat the large number of tasks used during meta-training as being sampled \iid from an underlying task distribution under the assumption that the number of task samples hasn't exceeded the total number of underlying classes.

\subsection{Concluding Remarks}

In the first part of the dichotomy, we show using Theorem~\ref{thm:case-1} that when the number of underlying classes is smaller than $nN$, it is highly unlikely for the train and test classes to be completely disjoint.   

In the second part of the dichotomy, we show using Theorem~\ref{thm:case-2} that when the number of underlying classes is larger than $nN$, it is unlikely to observe only a few (very small fraction of $nN$) unique train classes --- in fact the number of unique train classes observed would roughly speaking scale linearly with the number of task samples $N$.   

\textbf{Conclusion on why current FSL benchmarks target OOD.} Note that in current FSL benchmarks \textbf{i)} there is no overlap of classes observed in the train and test tasks; and \textbf{ii)} the number of train (base) classes observed (\eg 64 for \mini) is much smaller than the total number of train tasks (\eg $\approx 10^6$ for \mini).   
Thus, the two sides of the dichotomy (above) when taken together leads us to reject the  hypothesis/assumption of \iid sampled train and test tasks in the current FSL benchmarks.  
\newpage

\section{Overview of Notations for \texorpdfstring{$(S, Q)$}{SQ} Sampling in ID and OOD Benchmarks}
\label{app:sample-notation}


\vspace{1em}

\begin{table}[h!!]
    \setlength\extrarowheight{2pt}
    \centering
    \caption{An overview of the notations used to describe each of the two steps: i) sampling the task from the training/test task distribution, and ii) sampling (S,Q) pair from the task; for the OOD benchmarks (\mini, \cif, \FC, \tiered, Zappos-OOD) and ID benchmarks (FEMNIST, Zappos-ID).}
    \vspace{0.25em}
    \begin{tabular}{c|c|c|c}
    \multicolumn{2}{c|}{Benchmark / Steps} &  \textbf{Step 1:} $\calT \sim \Prob(\calT)$  &   \textbf{Step 2:} $(S,Q) \sim \calT$  \\ \hline\hline
    \mini, \cif, \FC, 
        & Train 
        & $\calT_{\bc_B} \sim \Prob_{\calC_B} \coloneqq \Unif(\{\calT_{\bc_B}: \bc_B \in [\calC_B^n]\})$
        & $S, Q \sim \calT_{\bc_B}$ \\ \cline{2-4}
    \tiered (OOD)
        & Test 
        & $\calT_{\bc_N} \sim \Prob_{\calC_N} \coloneqq \Unif(\{\calT_{\bc_N}: \bc_N \in [\calC_N^n]\})$
        & $S, Q \sim \calT_{\bc_N}$ \\ \hline
    Zappos-OOD 
        & Train 
        & $\calT_{\ba} \sim \Prob_{\calA}(\calT) = \Unif(\{\calT_{\ba}: \ba \in [\calA^2])\}$
        &  $S, Q \sim \calT_{\ba}$ \\ \cline{2-4}
    (OOD) 
        & Test 
        & $\calT_{\ba'} \sim \Prob_{\calA'}(\calT) = \Unif(\{\calT_{\ba'}: \ba' \in [\calA'^2])\}$
        & $S, Q \sim \calT_{\ba'}$ \\ \hline\hline
    FEMNIST 
        & Train 
        & $\calT_{\id} \sim \Prob(\id)$
        & $S, Q \sim \calT_{\id}$ \\ \cline{2-4}
    (ID) 
        & Test 
        & $\calT_{\id} \sim \Prob(\id)$
        & $S, Q \sim \calT_{\id}$ \\ \hline
    Zappos-ID 
        & Train 
        & $\calT_{\ba} \sim \Prob_{\calA}(\calT) = \Unif(\{\calT_{\ba}: \ba \in [\calA^2])\}$
        & $S, Q \sim \calT_{\ba}$ \\ \cline{2-4}
    (ID) 
        & Test 
        & $\calT_{\ba} \sim \Prob_{\calA}(\calT) = \Unif(\{\calT_{\ba}: \ba \in [\calA^2])\}$
        & $S, Q \sim \calT_{\ba}$ \\ \bottomrule
    \end{tabular}
    \label{tab:sampling-steps-table}
\end{table}


\section{Dataset Preprocessing and Hyperparameter Details}
\label{app:zappos-and-hyperparameter}

Here we first provide some details on the logic used to construct the ID benchmark Zappos-ID and its OOD counterpart Zappos-OOD.  We then list the set of hyperparameter configurations used to train the meta-learning methods \pn, \ridge, \svm, \fomaml and the supervised learning baselines \textsc{mb}, \textsc{sb} and \textsc{ic} on each of the benchmarks in the paper. 

\subsection{Zappos Preprocessing}

Recall that the Zappos dataset is motivated through an online shopping recommendation problem, where each task is a binary classification of shoe images into an attribute context. Every online user is represented by such a task, where the user's preference for shoes is specified by the corresponding shoe attribute context.
We consider a simplified setting where we fix a set of universal shoe attributes $\calA$ and each user's preference is specified exactly by a pair of attributes $\ba = (a_1, a_2) \in \calA^2,\; a_1\neq a_2$. The Zappos-ID and Zappos-OOD FSL benchmarks we use are derived from the UT Zappos50k corpus which consists of 50,025 shoe images each annotated with a list of attributes the shoe possess. 

\textbf{Attributes Selection.} We limit the subset of attributes we consider to the 78 considered by \citet{ren2020flexible} (Table 7).
Recall that the task distribution we consider is the uniform distribution $\calT_{\ba} \sim \Prob_{\calA}(\calT) = \Unif(\{\calT_{\ba}: \ba \in [\calA^2])\}$ (Table~\ref{tab:sampling-steps-table}) over all non-repeating attribute pairs in $\calA$.
In order to ensure that each attribute pair in $\calA$ has at least $20$ shoes carrying both the attributes (feasible pairs), we only consider the uniform distribution over such feasible attribute pairs.
Thus, we reduce the original set of 78 attributes to 66, since 12 of the attributes where found to be infeasible with every other attribute in the original set.

\textbf{Determining $\calA, \calA'$.} For the Zappos-ID benchmark we use the set of attributes specified by $\calA$ (of size 36) to \iid sample 1000 (or 50) train and 25000 test tasks. On the other hand, as mentioned in Section~\ref{sec:OOD} we use a disjoint set of attributes $\calA'$ (of size 30) to sample 25000 test tasks for Zappos-OOD (see Table~\ref{tab:zappos-attributes} for the exact sets). 
To determine this partition, we first consider a graph of 66 nodes, where each node represents an attribute and an undirected edge between a pair of attribute node is weighted by the number of shoe images (in the corpus) that have both attributes.
Using spectral clustering, we find an approximate min-cut bipartition of this graph. In other words, we partition the entire set of attributes into two subsets in a way that reduces the number of images which carry pairs of attributes that are not in the same subset.
This graph partition gives us the split of a 36-attribute set $(\calA)$ and a 30-attribute set $(\calA')$. 




\vspace{2em}
\begin{table}[!ht]
    \centering
    \scriptsize
    \setlength\tabcolsep{3pt}
    \vspace{0.25em}
    \caption{We show the disjoint set of attributes $\calA, \calA'$ for the Zappos-ID/OOD datasets. For the Zappos-ID dataset we use the set of attributes in $\calA$ to \iid sample train and test tasks $\calT_{\ba} \sim \Prob_{\calA}$. For the Zappos-OOD dataset the train tasks are still sampled using $\calA$ \ie $\calT_{\ba} \sim \Prob_{\calA}$ but the test tasks are sampled using $\calA'$ \ie $\calT_{\ba'} \sim \Prob_{\calA'}$.}
    \vspace{0.5em}
    \begin{tabular}{p{0.05\textwidth}|llll}
\multirow{9}{*}{$\calA$} 
 & Category.Boots & Category.Sandals & Closure.Ankle.Strap & Closure.Ankle.Wrap \\ 
 & Closure.Buckle & Closure.Bungee & Closure.Button.Loop & Closure.Elastic.Gore \\ 
 & Closure.Pull.on & Closure.Sling.Back & Closure.Snap & Closure.T.Strap  \\
 & Closure.Toggle & Closure.Zipper & Gender.Girls & Gender.Women  \\
 & HeelHeight.High.heel & HeelHeight.Short.heel & Material.Rubber & Material.Suede \\  
 & SubCategory.Ankle & SubCategory.Clogs.and.Mules & SubCategory.Flats & SubCategory.Heel \\  
 & SubCategory.Heels & SubCategory.Knee.High & SubCategory.Mid.Calf & SubCategory.Over.the.Knee \\  
 & ToeStyle.Almond & ToeStyle.Center Seam & ToeStyle.Closed Toe & ToeStyle.Open Toe  \\
 & ToeStyle.Peep Toe & ToeStyle.Pointed Toe & ToeStyle.Round Toe & ToeStyle.Snip Toe  \\ \midrule

\multirow{8}{*}{$\calA'$} 
 & Category.Shoes & Category.Slippers & Closure.Hook.and.Loop & Closure.Lace.up  \\
 & Closure.Monk.Strap & Closure.Slip.On & Gender.Boys & Gender.Men  \\
 & Material.Corduroy & Material.Silk & Material.Wool & SubCategory.Boat.Shoes  \\
 & SubCategory.Crib.Shoes & SubCategory.Firstwalker & SubCategory.Loafers & SubCategory.Oxfords  \\
 & SubCategory.Prewalker & SubCategory.Slipper.Flats & SubCategory.Sneakers.and.Athletic.Shoes & ToeStyle.Algonquin \\ 
 & ToeStyle.Apron Toe & ToeStyle.Bicycle Toe & ToeStyle.Bump Toe & ToeStyle.Capped Toe  \\
 & ToeStyle.Medallion & ToeStyle.Moc Toe & ToeStyle.Snub Toe & ToeStyle.Square Toe  \\
 & ToeStyle.Wide Toe Box & ToeStyle.Wingtip \\
    \end{tabular}
    \label{tab:zappos-attributes}
\end{table}

\subsection{Hyperparameter settings}





For all the experiments performed in our paper, we have run grid search to tune both the meta-learning method-specific hyperparameters and the optimization hyperparameters for each meta-learning method. For the existing OOD FSL benchmarks, we found that the best hyperparameters are often the same as what was reported in the original paper. Additionally, the absolute performance and performance orders we report on OOD FSL benchmarks match with other works after the hyperparameter tuning -- indicating that we have been fair in representing each meta-learning method with its best hyperparameter setting. For our newly identified in-distribution benchmarks, we took care in tuning these hyperparameters for each method to ensure fairness of comparison.

In Table~\ref{tab:hyper-parameter-details} we list the optimization and other algorithm-specific hyperparameters for each meta-learning method, dataset pair.
For \textbf{optimization hyperparameters}, we describe
\begin{enumerate}[label=(\alph*),topsep=0pt,itemsep=-1ex,partopsep=1ex,parsep=1ex]
    \item the total number of epochs (each with 1000 iterations of gradient updates except for FEMNIST, Zappos-ID, and Zappos-OOD where each epoch depends on number of training tasks);
    \item step (staircase) learning rate schedule (lr: $e_1(r_1)-e_2(r_2)-\ldots-e_n(r_n)$ where $r_i$ is the value of the learning rate and $e_i$ is the epoch number at which $r_i$ is first set);
    \item the number of tasks in a minibatch to compute one gradient update (task batch size or task BS);
    \item for all experiments we use SGD optimizer (Nesterov Momentum 0.9).
\end{enumerate}
     
\textbf{Other meta-learning method specific hyperparameters}:
The scale-factor refers to a constant factor that is multiplied to the logits for each class, before passing them through softmax. In some cases these are fixed through training, and in others they are learnable. For other method-specific hyperparameters that we borrow directly from the original paper, we provide the references.
\begin{enumerate}[label=\roman*),topsep=0pt,itemsep=-1ex,partopsep=1ex,parsep=1ex]
    \item During meta-training, we perform the same \textbf{data augmentation} used in \citet{chen2019closer} for \minimod and used in \citet{lee2019meta} for \cifmod, \FCmod, \tieredmod. For FEMNIST, Zappos-ID and Zappos-OOD we do not perform any data augmentations.
    \item We use a \textbf{weight decay} of $5e-4$ for all datasets except for FEMNIST for which a weight decay of $1e-2$ is used to prevent overfitting.
    \item We use the Resnet-12 \textbf{backbone} for all our experiments except for FEMNIST which is made up of the relatively easier tasks of digit classification. For FEMNIST, we use a four layer Conv-64 model backbone.
\end{enumerate}

\textbf{Computational Resources.} For all experiments we use (at most) four NVIDIA GEFORCE GTX 1080Ti GPU cards. A single run of \pn, \ridge, \svm, and \fomaml on the \minimod dataset takes $\approx$ 12 hrs, 48 hrs, 48 hrs, and 72 hrs respectively. For experiments on Zappos and FEMNIST, except for \fomaml which takes about 24 hrs, all the other experiments take no more than 5 hrs to complete training.


 
 

\begin{table}
    \centering
    \small
    \caption{Hyperparameter details for different algorithms and datasets in Sections~\ref{sec:ID},~\ref{sec:OOD}.}
    \vspace{0.75em}
    \begin{tabular}{p{8em}|p{13em}|p{19em}}
        Alg / Dataset &  Optimization hyperparameters & Other hyperparameters\\\toprule
        
        \pn / \newline (\minimod, \cifmod, \tieredmod) &   60 Epochs \newline lr: 0(0.1)-20(6e-3)-40(1.2e-3)-50(2.4e-4) \newline task BS: 4(5-way), 1(>5-way) & scale-factor 10 \newline euclidean metric \cite{snell2017prototypical} \\ \hline
        
        \pn/ \newline \FCmod &  15 Epochs \newline lr: 0(0.1)-5(6e-3) \newline task BS: 4(5-way), 1(>5-way) & scale-factor 10 \newline euclidean metric \cite{snell2017prototypical} \\ \hline
        
        (\ridge, \svm)/ \newline (\minimod, \cifmod, \tieredmod) & 60 Epochs \newline lr: 0(0.1)-20(6e-3)-40(1.2e-3)-50(2.4e-4) \newline task BS: 8 (always 5 way) & learnable scale \newline other as in \citet{lee2019meta}
        \\ \hline
        
        (\ridge, \svm)/ \newline \FCmod &  30 Epochs \newline lr: 0(0.1)-20(6e-3) \newline task BS: 8 (always 5 way) & scale-factor 7 \newline other as in \citet{lee2019meta} \\ \hline
    
        \fomaml/ \newline (\cifmod, \FCmod) &  60 Epochs \newline lr: 0(0.01)-20(6e-3)-40(1.2e-3) \newline task BS: 4 (always 5 way) & scale-factor 1 \newline inner loop step size: 0.01 \newline inner loop steps: 5 (train), 20 (test) \newline other as in \citet{finn2017meta} \\ \hline
    
        \fomaml / \newline (\minimod, \tieredmod) &   70 Epochs \newline lr: 0(1e-2)-35(1e-3)-65(1e-4) \newline task BS: 4 (always 5 way) & scale-factor 1 \newline inner loop step size: 0.01 \newline inner loop steps: 5 (train), 20 (test) \newline other as in \citet{finn2017meta} \\ \hline
        
        (\pn, \ridge, \svm, \fomaml)/ \newline (Zappos-ID, Zappos-OOD)&  60 Epochs \newline lr: 0(0.1)-30(6e-3) \newline task BS: 4 & \pn: scale-factor 10, euclidean metric \cite{snell2017prototypical}  
        \newline \ridge: learnable scale
        \newline \svm: learnable scale
        \newline \fomaml: inner loop steps: 5 (train), 20 (test), other as in \citet{finn2017meta}
        \\ \hline
        
        (\pn, \ridge, \svm, \fomaml)/ \newline FEMNIST &  100 Epochs \newline lr: 0(1e-3)-50(1e-4) \newline task BS: 5 & \pn: scale-factor 10, euclidean metric \cite{snell2017prototypical}  
        \newline \ridge: learnable scale
        \newline \svm: learnable scale
        \newline \fomaml: inner loop steps: 5 (train), 20 (test), other as in \citet{finn2017meta}
        \\ \hline
        
        (\textsc{sb}, \textsc{mb})/\newline (\minimod, \cifmod, \tieredmod, \FCmod) &  100 Epochs \newline lr: 0(0.1)-90(1e-2) \newline batch size: 128 & \textsc{sb}, \textsc{mb}: We project features in to unit norm during meta-train and use euclidean metric with scale-factor 10 in meta-test  \cite{snell2017prototypical} \newline
        \textsc{mb} only: we further finetune (lr=1e-3) on the meta-learning objective for additional 30 epochs.
        \\ \hline 
    
        \textsc{ic}/ \newline (\cifmod, \minimod) &  100 Epochs \newline lr: 0(0.1)-90(1e-2) \newline batch size: 128 & Figure~\ref{fig:model-selection}(b) (\cifmod): For supervised learning image classification (IC) baseline we collect all the images belonging to the base classes in \cifmod and randomly split them into 80\%(train)-10\%(val)-10\%(test). The val and test rankings used for computing the Kendall coefficient ($\rho$) are obtained using the val and test splits.\newline Figure~\ref{fig:conclusion_inconsistency}(b) (\minimod): The supervised learning IC baseline is trained on 600 train images from each base class in \minimod. After training is completed, we identify two IC models that differ by $\epsilon=0.5\%$ in terms of their generalization performance over a test set made up of all the unused examples ($\sim$ $700$) from each base class. To test the frequency of conclusion flips, for each of the 100 comparison runs, a random test subset is sampled from this test set with 100 examples sampled from each class. The chosen IC model pair is then evaluated over this test subset and their performance difference is recorded for this comparison run. \\ \hline

    
    \end{tabular}
    \label{tab:hyper-parameter-details}
\end{table}

\clearpage

\section{Additional Results on Evaluating ID Performance}
\label{app:additional-id}


In this section, we first present additional results on the choice of a different meta-learning method or the usage of a different $(S,Q)$ sampling strategy (\fixml) for the same meta-learning method can lead to improvements in OOD performance at the cost of ID performance. Then, we show how reducing the number of training tasks is unlikely to change the performance order of meta-learning methods in the ID benchmark Zappos-ID --- even though by reducing the number of training tasks, the number of unseen attribute pairs at test time increases. Finally, we present additional results how the degree of train/test task distribution mismatch can impact the performance order of meta-learning methods.

\subsection{Additional results on ID/OOD Performance Tradeoffs}

In the main text (Section~\ref{sec:ID}, Figure~\ref{fig:id-ood-tradeoffs}), we have compared the BaseGen and NovelGen performance of the best validation snapshots from different meta-learning methods (\pn, \svm, \ridge, \fomaml) and supervised trained baselines (\textsc{sb}, \textsc{mb}) on the \minimod and \tieredmod datasets.
Here, we show how the performance order of the same set of methods \textit{can also flip} in the ID (BaseGen) and OOD (NovelGen) settings \textbf{on two additional datasets}: \cifmod and \FCmod (Figure~\ref{fig:additional-id-plots}(a),(b)). Additionally, in the main text, we have shown how switching the choice of the $(S, Q)$ sampling strategy to one with fixed support sets (\fixml \citet{setlur2020support}) can also lead to improvements in the OOD performance at the cost of ID performance on the \cifmod dataset. In Figure~\ref{fig:additional-id-plots}(c),(d) we provide further evidence of this performance tradeoff \textbf{on two more datasets} \minimod and \FCmod.

\begin{figure*}[!ht]
  \begin{minipage}{.24\textwidth}
    \setlength{\abovecaptionskip}{1pt} 
    \includegraphics[width=\linewidth]{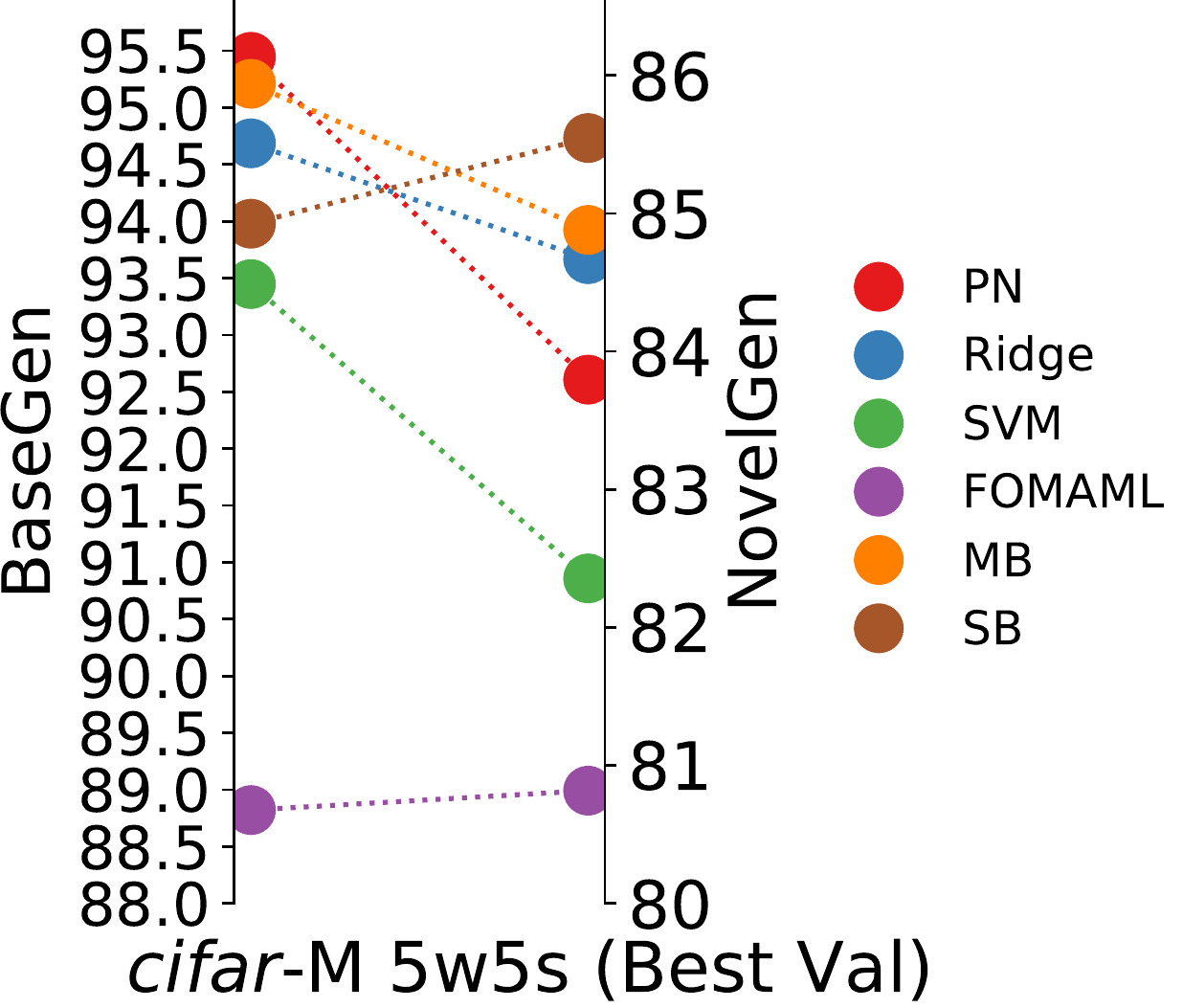}
    \caption*{(a)}
  \end{minipage}\hfill
  \begin{minipage}{.24\textwidth}
    \setlength{\abovecaptionskip}{1pt} 
    \includegraphics[width=\linewidth]{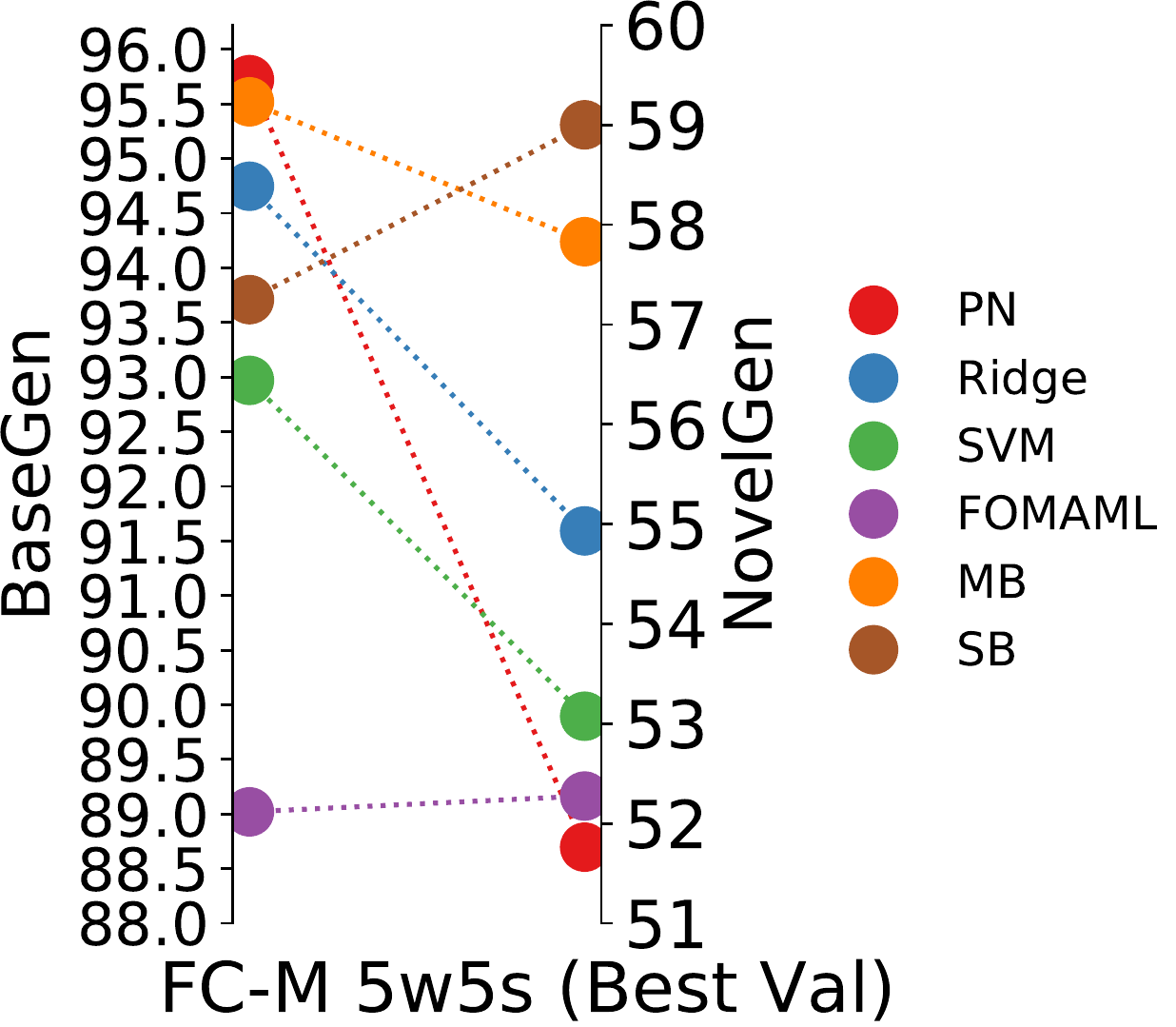}
    \caption*{(b)}
  \end{minipage}\hfill
  \begin{minipage}{.24\textwidth}
    \setlength{\abovecaptionskip}{1pt} 
    \includegraphics[width=\linewidth]{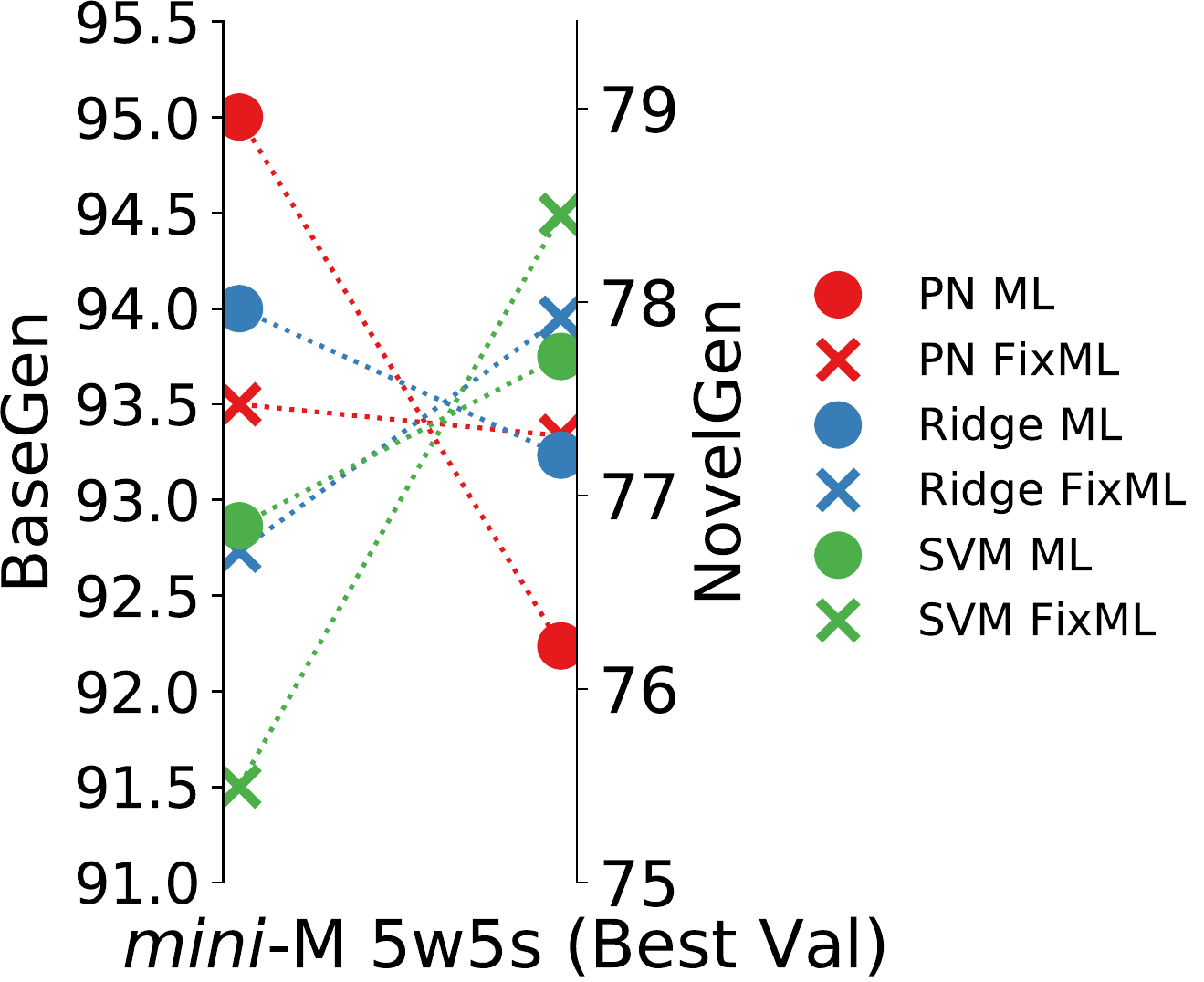}
    \caption*{(c)}
  \end{minipage}\hfill
  \begin{minipage}{.24\textwidth}
    \setlength{\abovecaptionskip}{1pt} 
    \includegraphics[width=\linewidth]{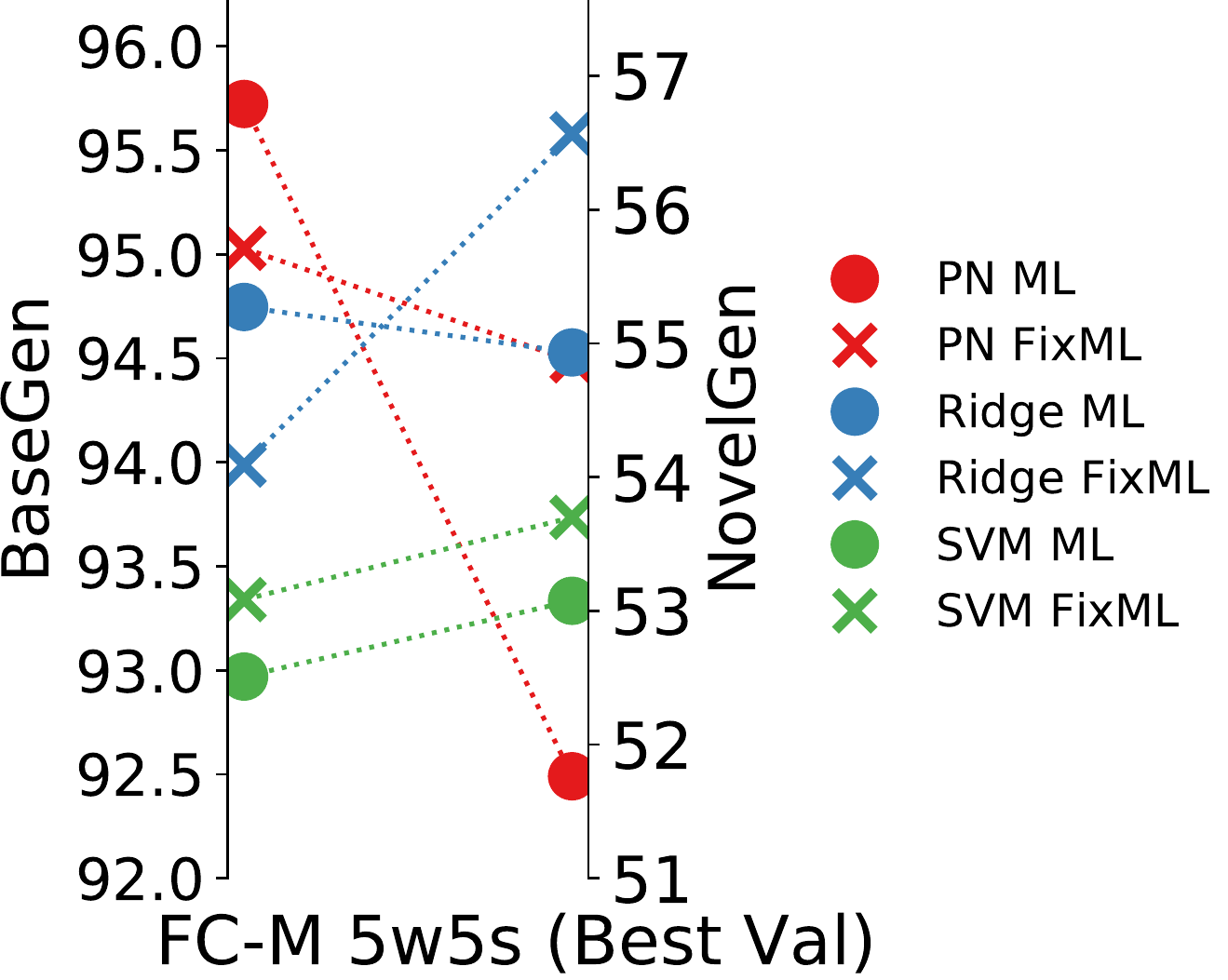}
    \caption*{(d)}
  \end{minipage} 
  \vspace{1em}
 \caption{BaseGen and NovelGen performance tradeoff (for best validation snapshots): over the choice of a set of four meta-learning and two supervised pre-training methods on \cifmod (a) and \FCmod (b); over the use of \fixml $(S, Q)$ generation strategy or not (\ml) with \svm, \ridge and \pn on \minimod in (c) and \FCmod in (d). \vspace{1em}}
 \label{fig:additional-id-plots}
\end{figure*}

\subsection{In-distribution Performance with Reduced Number of Training Tasks}

Now we analyze the impact of the number of training tasks on the performance order of meta-learning methods in ID scenario. For the 2w5s results on the Zappos-ID dataset in Table~\ref{tab:results-id} we used 1000 (or 50) training tasks and 25000 test tasks. Here, we also consider using 250 training tasks (while still using the same number of test tasks). 
As we have discussed in Section~\ref{sec:ID}, it is possible to encounter unseen classes (attribute pairs in the case of Zappos) at test time even while evaluating a meta-learning method in the ID scenario. When we reduce the number of training tasks, we observe more unseen attribute pairs at test time. For 250 train tasks we observe 153 and for 50 train tasks we observe 269 unseen attribute pairs at test time.

\begin{table}[!ht]
    \setlength\extrarowheight{2pt}
    \centering
    \caption{We analyze the 2w5s performance order of \pn, \svm, \ridge and \fomaml on the Zappos-ID dataset with reduced number (50, 250) of train tasks and compare the performance order (ranking in parentheses) observed with the larger set of 1000 train tasks.}
    \vspace{1em}
    \begin{tabular}{c|c|c||c} 
    Methods / \# Train tasks  & 50 & 250 & 1000 \\ \hline 
    PN & $^{\mathbf{(1)}}77.67 \pm 0.17\%$ &$^{\mathbf{(1)}}81.67 \pm 0.16\%$  & $^{\mathbf{(1)}}86.58 \pm 0.15\%$ \\
    Ridge & $^{\mathbf{(2)}}74.75 \pm 0.16\%$ & $^{\mathbf{(2)}}80.84 \pm 0.15\%$ & $^{\mathbf{(2)}}85.56  \pm  0.16\%$ \\
    SVM & $^{\mathbf{(3)}}74.06 \pm 0.17\%$ & $^{\mathbf{(3)}}80.15 \pm 0.17\%$ & $^{\mathbf{(3)}}85.12 \pm 0.16\%$ \\
    FO-MAML & $^{\mathbf{(4)}}69.85 \pm 0.18\%$ & $^{\mathbf{(4)}}73.20 \pm 0.16\%$  & $^{\mathbf{(4)}}80.14 \pm 0.15\%$ \\ \bottomrule
    \end{tabular}
    \label{tab:small-zappos}
\end{table}

In Table~\ref{tab:small-zappos} we can see that even with reduced training tasks and more unseen attribute pairs at test time the performance order of \pn, \ridge, \svm and \fomaml is retained. Note that the same order is observed on the other ID benchmark FEMNIST. This also matches the BaseGen performance order obtained after doing ID evaluations on modified FSL benchmarks. This result further confirms that the performance evaluations done on ID datasets are much more consistent than OOD datasets.

\subsection{Degree of Mismatch between Train \texorpdfstring{$(\Prob_{\calC_B})$}{PCB} and Test \texorpdfstring{$(\Prob_{\calC_N})$}{PCN} Distributions}
\label{app:subsec:degree-of-mismatch}

In Section~\ref{sec:ID} we notice that the performance order of meta-learning methods in the ID scenario (Zappos-ID) is quite different from that of the OOD FSL benchmarks (\eg \mini). We believe this is mainly because of the OOD nature of FSL benchmarks, which we also proof formally in Appendix~\ref{app:formal-analysis}.  
Moreover, since the type/degree of the mismatch between training and test distributions can vary for different FSL benchmarks, the performance order of popular methods is not as consistent as we would like them to be on these OOD benchmarks (see Section~\ref{subsec:inconsistencies-in-OOD}).

\textbf{Range of test task distributions.} To further analyze the impact of the degree of distribution mismatch on the performance order, for each FSL benchmark we construct \textbf{a range of new task distributions that are increasingly dissimilar to the train task distribution} $\Prob_{\calC_B}$ and \textbf{similar to the test task distribution} $\Prob_{\calC_N}$. 
We denote this set of task distributions as $\{\pbnlambda, \lambda \in [0,1]\}$. 
When sampling an $n$-way task from  $\pbnlambda$, instead of performing the first step of sampling outlined in Table~\ref{tab:sampling-steps-table}, we sample the class tuple $\bc = (c_1, \ldots, c_n) \in [(\calC_B \cup \calC_N)^{n}]$ where $c_i$'s are sequentially sampled non-repeatedly from $\calC_N$ uniformly with probability $\lambda$ and from $\calC_B$ with probability $1-\lambda$. The second step of sampling $(S, Q)$ from the task distribution $\calT_{\bc}$ is done in the usual way. It is clear that $\Prob_{\scriptscriptstyle \textrm{BN}}(0)=\Prob_{\calC_B}$ and $\Prob_{\scriptscriptstyle \textrm{BN}}(1)= \Prob_{\calC_N}$. In addition, when $\lambda = |\calC_N|/(|\calC_N| + |\calC_B|)$, $\Prob_{\scriptscriptstyle \textrm{BN}}(\lambda)=\Prob_{\calC_B \cup \calC_N}$ which describes a continual learning setting where the evaluation task is made up of classes uniformly sampled from the union of base and novel classes.
Notice that the task distribution $\pbnlambda$ is not the same as the mixture distribution $\lambda \cdot \Prob_{\calC_N} + (1 - \lambda) \cdot\Prob_{\calC_B}$ because a single task from $\pbnlambda$ can have both classes from $\calC_N$ and $\calC_B$. With this set of new task distributions defined, we evaluate our learned algorithm snapshot for each meta-learning method not only over the training ($\Prob_{\calC_B}$) and the test ($\Prob_{\calC_N}$) task distribution but also over distributions from this interpolation set. We plot the performances in  Figure~\ref{fig:interpolation-plots}.

\begin{figure*}[h]
  \begin{minipage}{.24\textwidth}
    \setlength{\abovecaptionskip}{1pt} 
    \includegraphics[width=\linewidth]{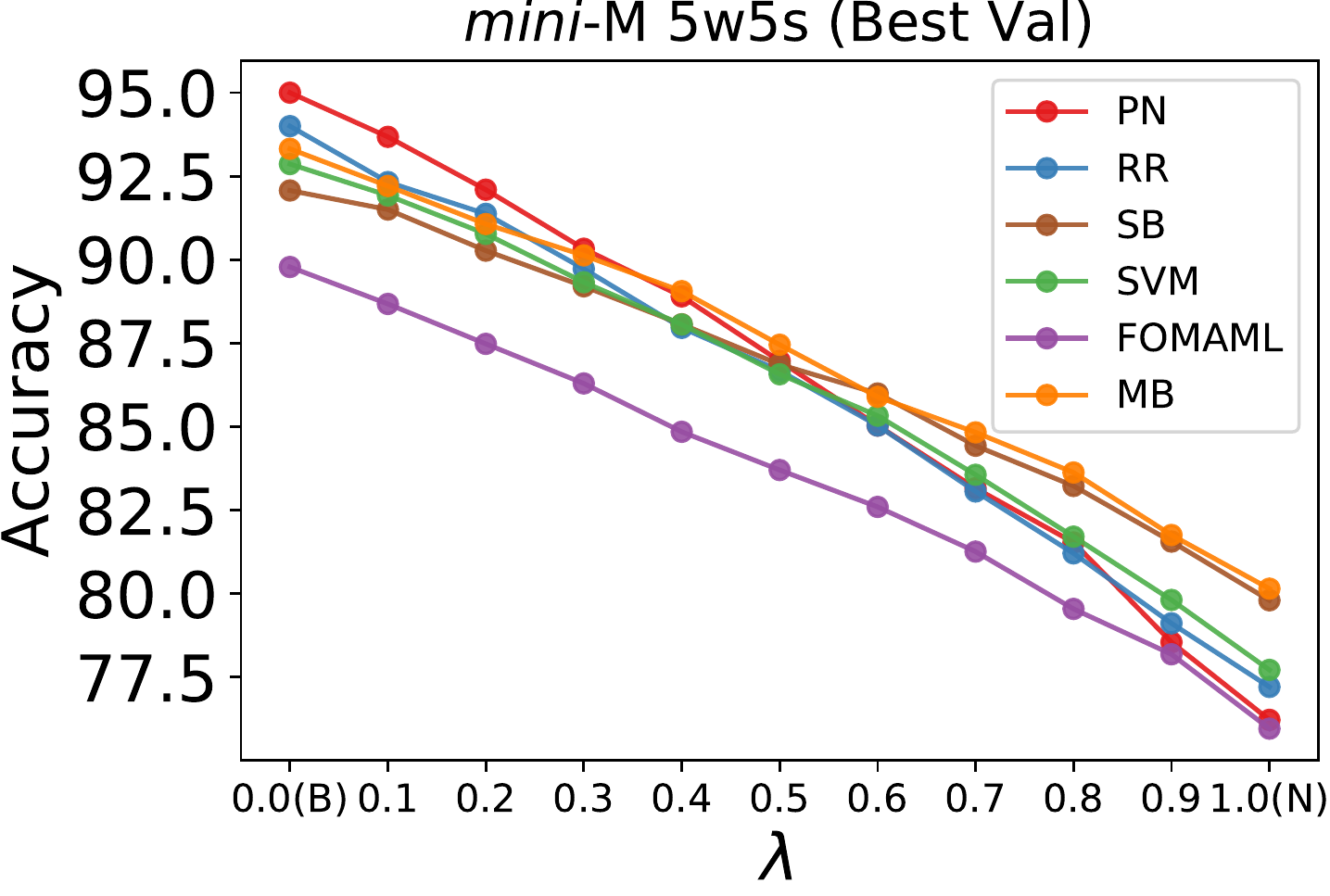}
    \caption*{(a)}
  \end{minipage}\hfill
  \begin{minipage}{.24\textwidth}
    \setlength{\abovecaptionskip}{1pt} 
    \includegraphics[width=\linewidth]{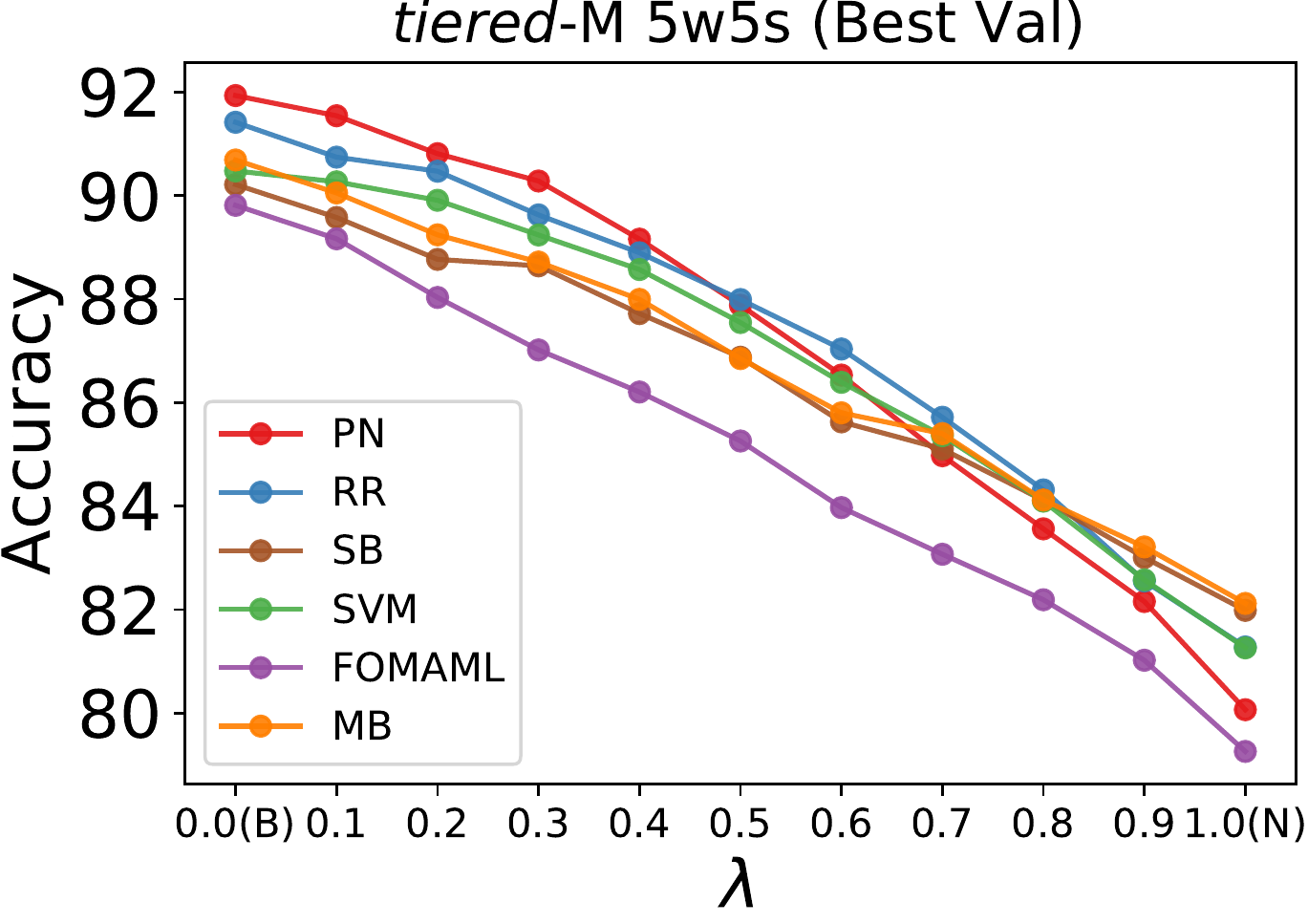}
     \caption*{(b)}
  \end{minipage}\hfill
  \begin{minipage}{.24\textwidth}
    \setlength{\abovecaptionskip}{1pt} 
    \includegraphics[width=\linewidth]{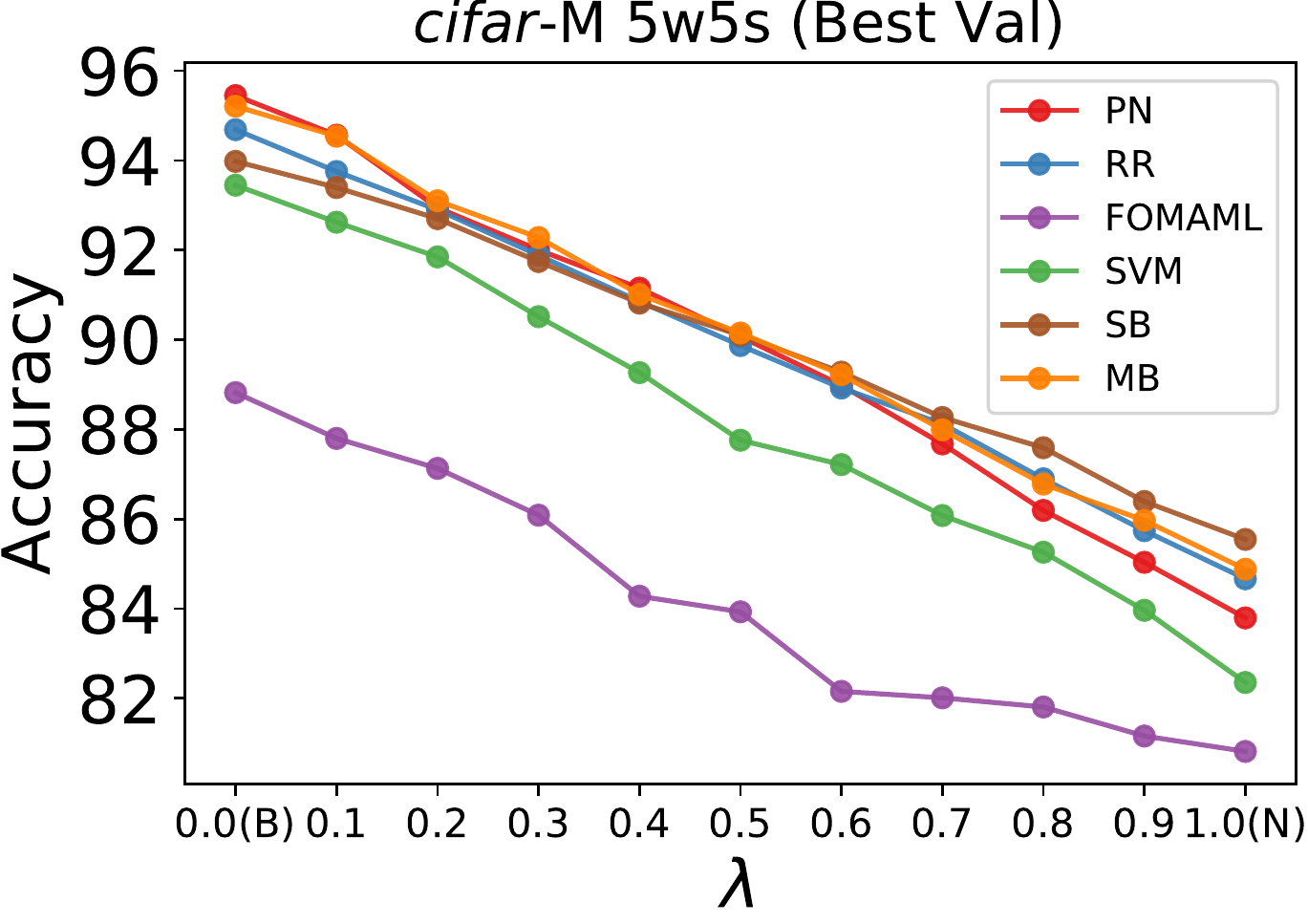}
    \caption*{(c)}
  \end{minipage}\hfill
  \begin{minipage}{.24\textwidth}
    \setlength{\abovecaptionskip}{1pt} 
    \includegraphics[width=\linewidth]{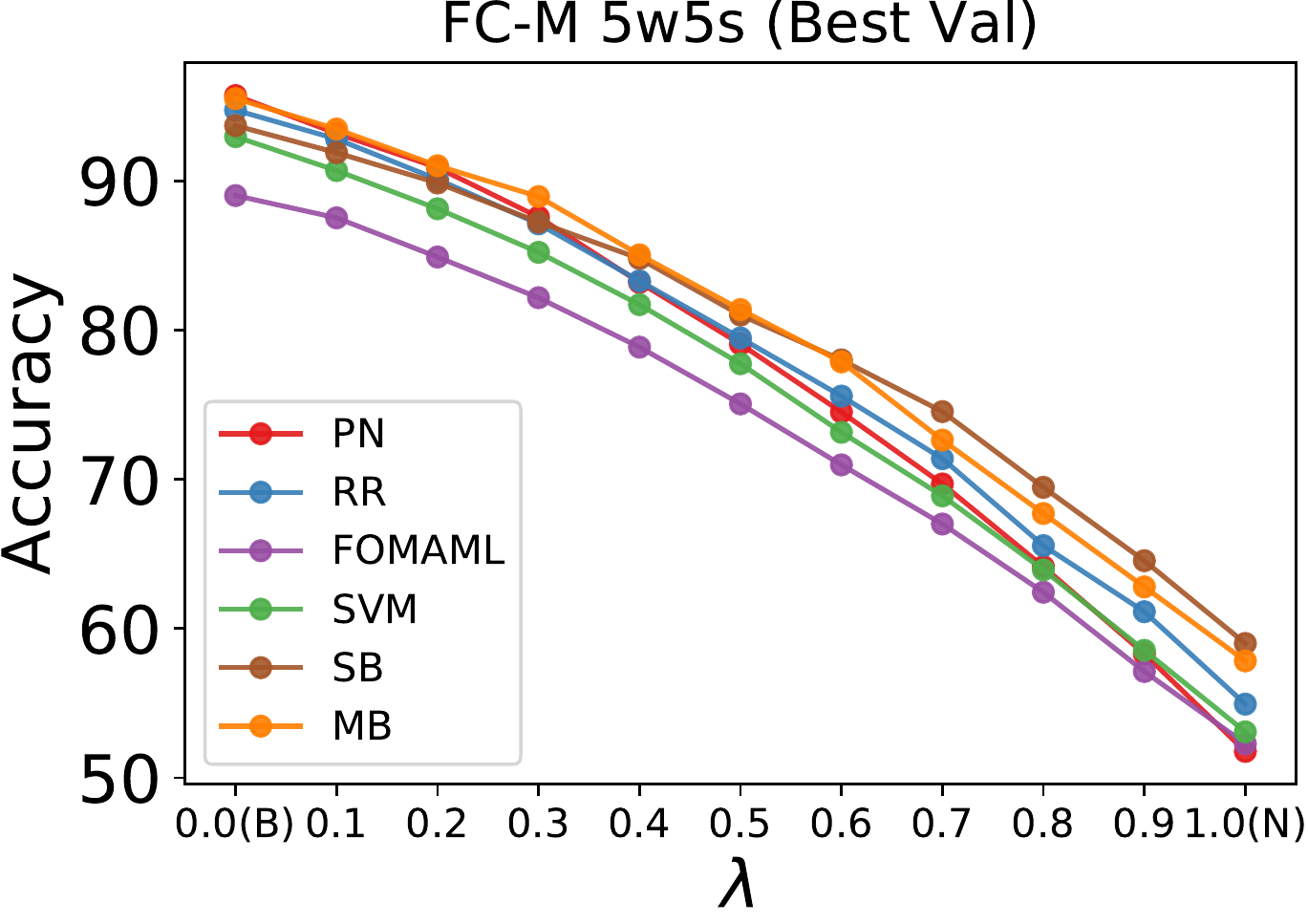}
\caption*{(d)}
  \end{minipage}
  \vspace{1em}
 \caption{Comparison of \pn, \svm, \ridge, \fomaml, \textsc{sb} and \textsc{mb}'s performance (best validation snapshots) on the set of distributions $\{\pbnlambda\}$ for (a) \minimod, (b) \tieredmod, (c) \cifmod, and (d) \FCmod datasets.\vspace{1em}}
 \label{fig:interpolation-plots}
\end{figure*}

\textbf{Performance order depends on degree of mismatch.} In Figure~\ref{fig:interpolation-plots} we first notice that the performance drops in a monotonically non-increasing way as $\lambda$ increases for each dataset/method combination. More importantly, we note that the performance order of the methods depends on the degree of mismatch. If the test task distribution is very similar to the train task distribution, then the performance order is mostly retained as the ID performance order (\eg for $\lambda \le 0.2$ in Figure~\ref{fig:interpolation-plots}(a)). Also in most cases the lines don't cross each other more than once, which indicates that if the degree of mismatch crosses a certain threshold then it is unlikely for the performance order to switch again for a given pair of methods. However, without seeing any test tasks during meta-training, it is difficult to know the degree of training and test task distribution mismatch, which makes it more difficult to predict when a given method will start performing better/worse. For example, the performance order changes at a lower value of $\lambda$ for \mini and \cif, as compared to \tiered.


\section{Additional Discussion and Results on OOD Evaluation}
\label{app:additional-ood}

\subsection{Simplified Example of Snapshot Selection and Hyperparameter Selection}
\label{app:subsec:snapshot-hyperparameter-selection-example}
To further explain our existing definition of hyperparameter and snapshot selection in Section~\ref{subsec:model-selection}, we accompany our original definitions with a concrete simplified scenario: suppose we want to train Prototypical Network (\pn) with two hyperparameter configs: training with learning rate 1e-3 vs. with learning rate 1e-4 (with all other hyperparameters the same).

\textbf{[Snapshot Selection]} Training under each hyperparameter config would generate a sequence of algorithm snapshots (most often with one snapshot saved after each training epoch). The term snapshot selection refers to the procedure of choosing one snapshot from each hyperparameter config: one from the 1e-3 learning-rate optimized \pn algorithm snapshot trajectory, and one from the 1e-4 learning-rate optimized \pn algorithm snapshot trajectory. There can be multiple strategies for snapshot selection, for example, picking \textbf{1)} the last snapshot at the end of training; \textbf{2)} the snapshot that has the lowest training loss; \textbf{3)} the snapshot that has the best BaseGen performance; or \textbf{4)} the snapshot that has the best ValGen performance.

\textbf{[Hyperparameter Selection]} Once an algorithm snapshot is chosen for each considered hyperparameter config, we need to decide which hyperparameter config’s (lr 1e-3 or lr 1e-4) selected snapshot to choose to be evaluated on the test task distribution with its performance recorded as the meta-learning method \textsc{pn}’s performance. The procedure of deciding which hyperparameter config’s snapshot to choose is called the hyperparameter selection problem. Similar to snapshot selection, there can be multiple strategies for hyperparameter selection: e.g., choosing the hyperparameter config whose selected snapshot has the best BaseGen or the best ValGen performance.

\subsection{Distinction between Snapshot Selection and Early Stopping}
\label{app:subsec:early-stopping}
\textit{Early stopping is not the same as what we mean by snapshot selection in our paper}. In standard supervised learning, \textit{early stopping} involves keeping track of the performance over an \iid validation dataset and stopping training after the validation performance starts to consistently deteriorate. It mainly serves as an approach to avoid overfitting and to save unnecessary computations if one believes further training would never improve the test performance. In terms of deciding which snapshot to choose, early stopping is often equivalent to the strategy of selecting the snapshot that has the highest validation accuracy. In contrast, in our paper, the term \textit{snapshot selection} refers to the \textbf{general problem} of deciding which snapshot to select for a given training run, \textbf{as opposed to a specific choice of selection strategy}. As we have explained in Section~\ref{subsubsec:snapshot-selection-valgen}, the strategy of picking the snapshot with the best meta-validation performance (early stopping) might not be the best strategy in the OOD scenario. There exist other snapshot selection strategies (different from early stopping), such as picking the snapshot with the lowest training loss or the best BaseGen performance.

\subsection{Differences in NovelGen Performance When Using Best BaseGen vs. ValGen for Snapshot Selection}
\label{app:subsec:basegen-vs-valgen}

In Section~\ref{subsec:model-selection} we compare different model selection strategies in the out-of-distribution scenario where we show that in some cases BaseGen performance can track NovelGen (Figure~\ref{fig:model-selection}(c)) while in other cases the ValGen performance may be better correlated with NovelGen (Figure~\ref{fig:model-selection}(d)).
In  Table~\ref{tab:basegen-vs-valgen} we present NovelGen (test) performance results on two different datasets when the snapshot is chosen using the best BaseGen snapshot (row 1) vs the best ValGen snapshot (row 2), and compare them against the best possible NovelGen across all snapshots (row 3). Here we notice that there can be a \textbf{significant difference in the selected snapshot’s NovelGen performance if we use one selection strategy instead of another} and there isn't a single snapshot selection strategy that yields the best results for multiple OOD scenarios.

\begin{table}[!h]
    \centering
    \caption{NovelGen performance for i) \pn trained on \tieredmod; and ii) \svm trained on \cifmod evaluated using the snapshot chosen with best BaseGen/ValGen/NovelGen performance.}
    \vspace{0.25em}
    \begin{tabular}{ccc}
        Snapshot selection strategy & PN on \tieredmod (Figure~\ref{fig:model-selection}(c)) & SVM on \cifmod (Figure~\ref{fig:model-selection}(d)) \\ \toprule
        Best BaseGen & $80.45 \%$ & $80.49 \%$ \\
        Best ValGen & $80.05 \%$ & $82.65 \%$ \\
        Best NovelGen & $80.45 \%$ & $82.79 \%$ \\ \bottomrule
    \end{tabular}
    \label{tab:basegen-vs-valgen}
\end{table}

\subsection{What to Consider While Designing Model Selection and Comparison Strategies for Both ID and OOD Performance?}
\label{app:subsec:model-selection-ID-OOD}

In Section~\ref{subsec:model-selection} we discuss in detail the implications of different model selection strategies on the final NovelGen performance of the selected snapshot specific to current OOD FSL benchmarks. In more generic settings, it is likely that one may wish to design model selection methods when one cares about both ID and OOD performance. In such a case we believe that the optimal model selection and comparison strategy would depend on: i) the final comparison metric; ii) the type of task access for model selection.

\begin{enumerate}[label=\roman*),topsep=0pt,itemsep=-1ex,partopsep=1ex,parsep=1ex, leftmargin=15pt]
\item \textbf{Method comparison metric}: We first need to define a metric for the final evaluation of a meta-learning method’s selected algorithm snapshot. One way is to individually evaluate ID and OOD performance and record it as a 2-tuple. In this case, a meta-learning method is said to outperform another only if its selected snapshot is better in each component of the 2-tuple. In this case, it is very possible that there does not exist a meta-learning method that clearly outperforms all others in this metric as we have seen from Figure 1(a)(b). Another way is to evaluate the performance on a mixture task distribution, with the training and OOD test distributions weighted by fixed probability weights (as we do in Appendix~\ref{app:subsec:degree-of-mismatch}). As we have seen in Appendix~\ref{app:subsec:degree-of-mismatch} (Figure~\ref{fig:interpolation-plots}), different probability weighting of the two distributions can result in different conclusions of which meta-learning method works the best.

\item \textbf{Task access during model selection}: We also need to specify what type of task samples are available during model selection. While it is reasonable to assume there are additional fresh \iid samples from the training distribution (e.g. by holdout training set), it depends on the specific application to know whether there exist task samples from the OOD task distribution. Different scenarios may arise based on this: i) In federated learning applications, one might be allowed to evaluate meta-learned algorithm snapshots over a small sample of users (tasks) from the OOD user population before deployment; or ii) In case one does not have such OOD task samples, a proxy distribution (e.g. validation task distribution in FSL benchmarks) may still be available and \iid task samples drawn from it could be used for model selection. However, its utility would depend on how similar it is to the actual OOD distribution and we have seen examples in Figure~\ref{fig:model-selection}(a) that using samples from such proxy task distributions might also be suboptimal.
\end{enumerate}

We hope that our work advocates for more discussion and development of model comparison metrics and selection strategies while taking into these considerations described above.

\subsection{Hyperparameter Selection Strategies}
\label{app:subsec:additional-ood-hp}

In Section~\ref{subsec:model-selection} we discuss the distinction between snapshot selection and hyperparameter selection and how the former is called as a subroutine while determining the best snapshot to represent a specific hyperparameter configuration. As snapshot selection is a ubiquitous problem we focus on analyzing it and exploring alternative strategies (in the context of meta-learning) in the main paper. Motivated by the work of \citet{gulrajani2020search}, we now discuss some ways of performing hyperparameter selection specifically for settings similar to the current FSL benchmarks where the few-shot classification tasks are determined by a class tuple. 

Because the current FSL benchmarks have set aside a specific validation set of classes, it would appear that using tasks generated by these classes is the only option for hyperparameter selection. However, as mentioned in \cite{gulrajani2020search}, there are also other alternatives.

\textbf{Cross validated hyperparameter selection}. Instead of using a single set of validation classes, one could rely on cross validation. Here, for a given hyperparameter configuration, one can train multiple times, where each training run is done on tasks generated by a different subset of the training (base) classes (or the union of base and val classes) and the proxy performance is calculated on tasks generated by the remaining classes not used in training. Finally, the performance is averaged over all runs for the given hyperparameter configuration and the hyperparameter with the highest performance would be chosen. 

However, this approach would still require performing snapshot selection for each run, thus it remains unclear how to best perform snapshot selection in this case. Despite this, cross-validated hyperparameter evaluation could potentially be more reliable than using a single set of validation classes for hyperparameter selection (as done on the current FSL benchmarks) but it would also be more computationally expensive and we leave further investigation of this hyperparameter selection approach as future work.

\textbf{Allowing restricted oracle evaluation over the test task distribution.} Instead of completely forbidding access to the test task distribution, one can allow a limited number of test task distribution performance evaluations for a given meta-learning method. In this case it becomes the responsibility of designer of the meta-learning method to decide how to best distribute this fixed number of evaluations wisely over different hyperparameter and snapshot choices.



\subsection{Variance of performance \texorpdfstring{$A(\calC_N)$}{ACN} over randomly sampled \texorpdfstring{$\calC_N$}{CN}}
\label{app:subsec:additional-ood-var}

In Section~\ref{subsec:inconsistencies-in-OOD} Inconsistency example 1, we have shown how limited number of novel classes in the evaluation can lead to a high chance of conclusion flips. This high degree of unreliability stems from the variance of the performance estimator $A(\calC_N)$ which only uses a subset of the larger underlying class set $\calC_L$. 
In Figure~\ref{fig:acn-plots} we plot the histogram of the random variable $A(\calC_N)$ randomized over the choice of novel classes $\calC_N$. We notice that the performance standard deviation is $2.49\%$ on \mini and $3.1\%$ on \tiered. The standard deviation on \tiered is higher since the number of underlying classes $|\calC_L|$ is larger in \tiered ($=552$). When the variance of the novel accuracy is as high as what we have observed here, it becomes very hard to clearly determine the better meta-learning method. To alleviate this, we provide some actionable suggestions like choosing benchmarks with more base classes during training and more novel classes during evaluation (see Section~\ref{subsec:inconsistencies-in-OOD}). 

\begin{figure*}[h!]
    \hfill
  \begin{minipage}{.35\textwidth}
    \setlength{\abovecaptionskip}{1pt} 
    \includegraphics[width=\linewidth]{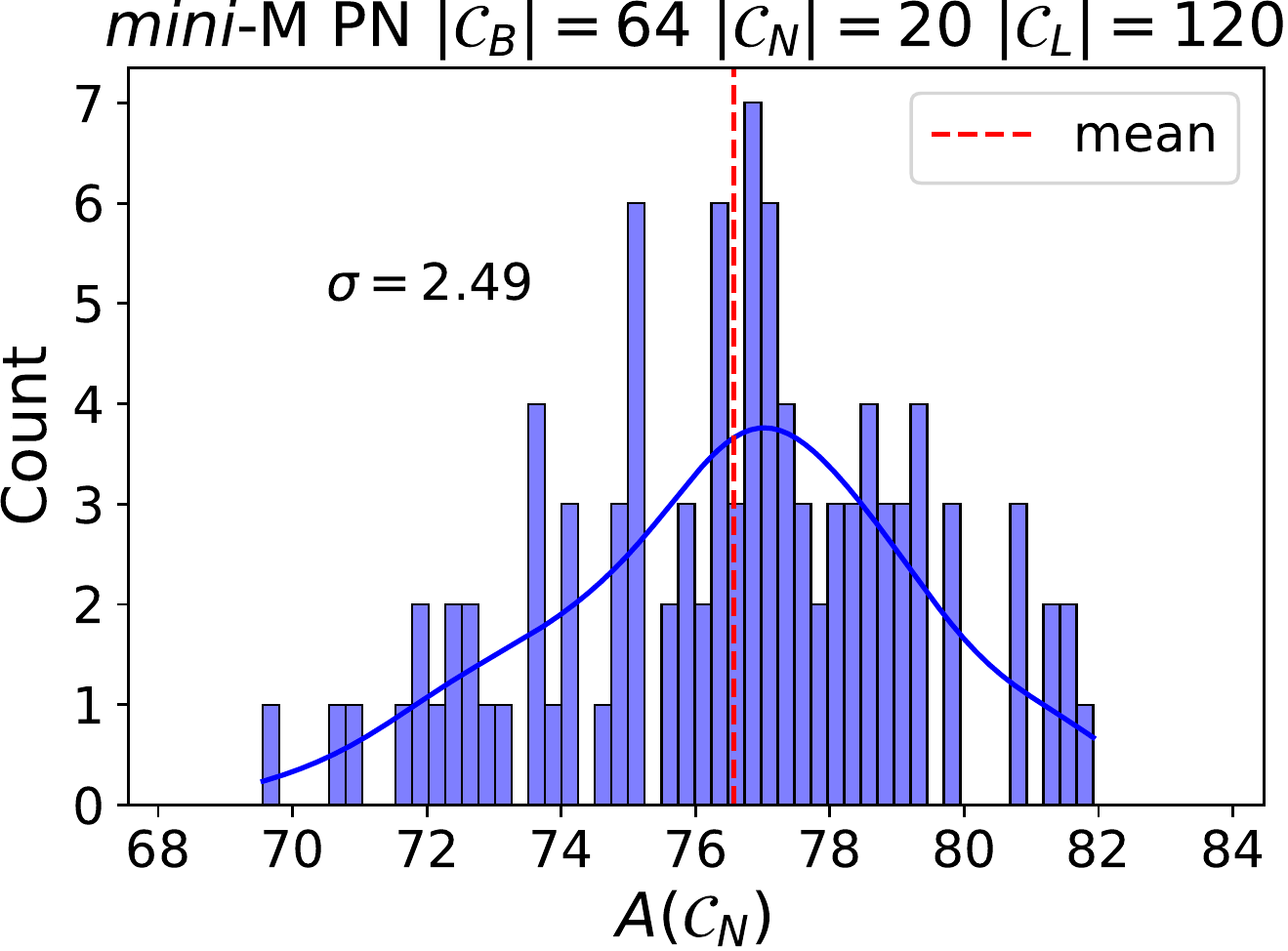}
    \caption*{(a)}
  \end{minipage}
  \begin{minipage}{.35\textwidth}
    \setlength{\abovecaptionskip}{1pt} 
    \includegraphics[width=\linewidth]{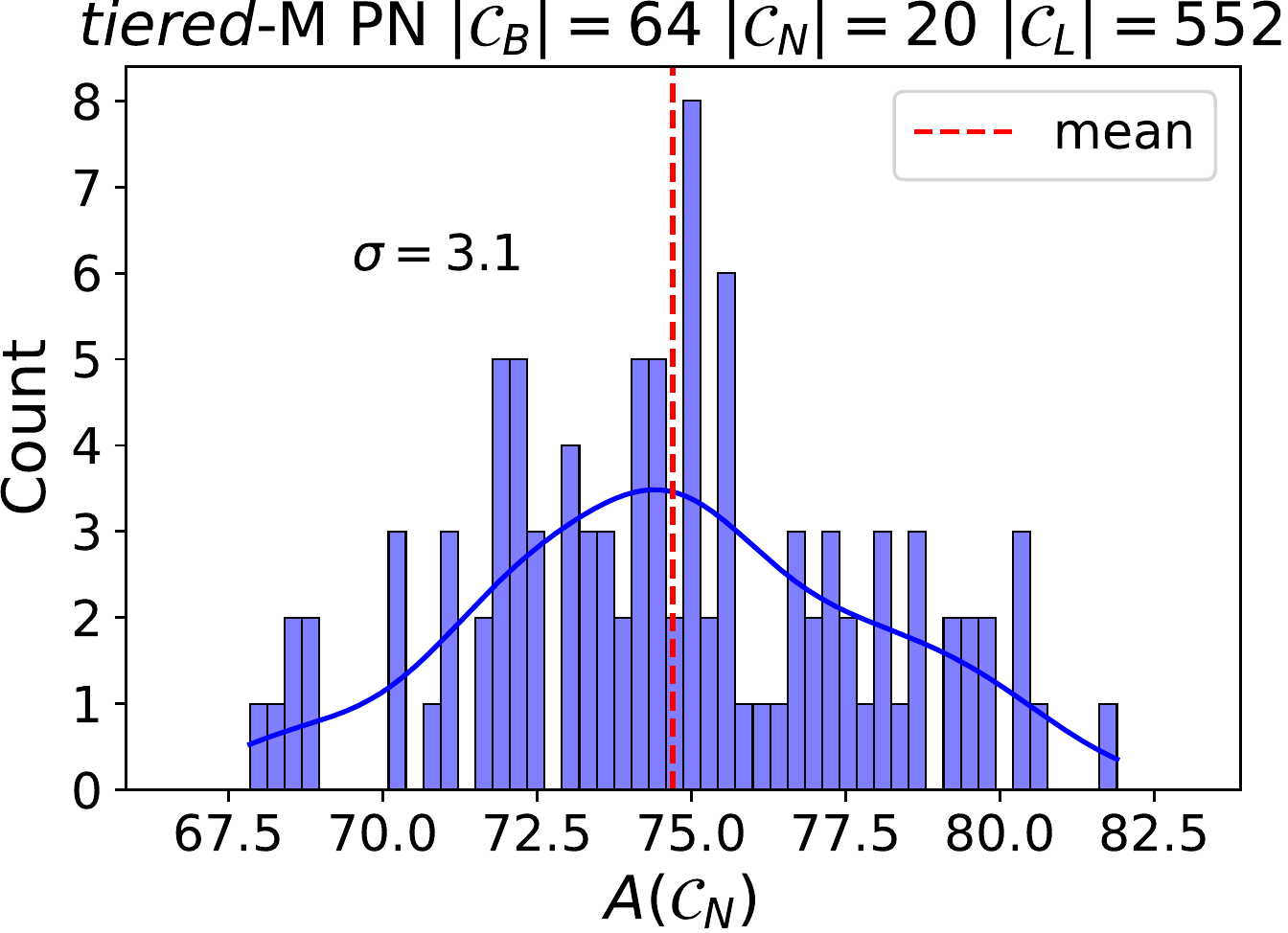}
     \caption*{(b)}
  \end{minipage}\hfill
  \vspace{1em}
 \caption{Histogram plots (over 100 runs) of the OOD accuracy $A(C_N)$ for  (a) a \pn trained on \mini and evaluated on random 20 out of 120 novel classes; and (b) a \pn trained on 64 \tiered base-classes (see Section~\ref{subsec:inconsistencies-in-OOD}) and evaluated on random 20 out of 552 novel classes.}
 \label{fig:acn-plots}
\end{figure*}

\end{document}